\documentclass[twoside,11pt]{article}

\usepackage[preprint]{jmlr2e}

\usepackage{lastpage}
\jmlrheading{}{}{1-\pageref{LastPage}}{}{}{}{Anqi Mao, Mehryar Mohri and Yutao Zhong}
\ShortHeadings{Ranking with Abstention}{Mao, Mohri, and Zhong}
\firstpageno{1}

\usepackage{microtype}
\usepackage{graphicx}
\usepackage{subcaption}
\usepackage{booktabs}

\usepackage{hyperref}

\usepackage{amsmath}
\usepackage{amssymb}
\usepackage{mathtools}
\usepackage{amsthm}

\theoremstyle{plain}

\theoremstyle{definition}

\theoremstyle{remark}

\usepackage{graphicx}
\usepackage[utf8]{inputenc} 
\usepackage[T1]{fontenc}    
\usepackage{url}            
\usepackage{booktabs}       
\usepackage{amsfonts}       
\usepackage{nicefrac}       
\usepackage{url}            
\usepackage{booktabs}       
\usepackage{amsfonts}       
\usepackage{nicefrac}       
\usepackage{microtype}      
\usepackage{nccmath}
\usepackage{times}
\usepackage{graphicx}
\usepackage{natbib}
\usepackage{algorithm}
\usepackage{algorithmic}
\usepackage{amsfonts}
\usepackage{amsmath}

\usepackage{mathtools}
\usepackage{amsthm}
\usepackage{graphicx}
\usepackage{enumerate}
\usepackage{algorithm}
\usepackage{algorithmic}
\usepackage{thm-restate}

\usepackage{url}
\usepackage{dsfont}
\usepackage[mathscr]{euscript}
\usepackage{booktabs}
\usepackage{makecell}
\usepackage{tabularx}
\usepackage{xcolor, colortbl}
\usepackage[normalem]{ulem}
\usepackage{soul}
\usepackage{prettyref}

\usepackage{multirow}

\usepackage{MnSymbol}
\DeclareMathAlphabet\mathbb{U}{msb}{m}{n}
\usepackage{xpatch}

\def\Rset{\mathbb{R}}

\DeclareMathOperator*{\E}{\mathbb E}

\DeclareMathOperator{\sign}{sign}

\DeclarePairedDelimiter{\abs}{\lvert}{\rvert} 
\DeclarePairedDelimiter{\bracket}{[}{]}
\DeclarePairedDelimiter{\curl}{\{}{\}}
\DeclarePairedDelimiter{\norm}{\|}{\|}
\DeclarePairedDelimiter{\paren}{(}{)}
\DeclarePairedDelimiter{\tri}{\langle}{\rangle}

\newcommand{\sC}{{\mathscr C}}
\newcommand{\sD}{{\mathscr D}}

\newcommand{\sH}{{\mathscr H}}

\newcommand{\sM}{{\mathscr M}}
\newcommand{\sP}{{\mathscr P}}

\newcommand{\sR}{{\mathscr R}}

\newcommand{\sX}{{\mathscr X}}
\newcommand{\sY}{{\mathscr Y}}

\newcommand{\sfL}{{\mathsf L}}

\newcommand{\ov}{\overline}

\newcommand{\wt}{\widetilde}
\newcommand{\e}{\epsilon}

\newcommand{\ignore}[1]{}

\newcommand{\1}{\mathds{1}}

\newcommand{\lrank}{{\sfL_{0-1}}}
\newcommand{\lrankbi}{{\wt \sfL_{0-1}}}
\newcommand{\lbi}{{\wt \sfL}}
\newcommand{\labs}{{\sfL_{0-1}^{\rm{abs}}}}
\newcommand{\labsbi}{{\wt \sfL_{0-1}^{\rm{abs}}}}

\hypersetup{
  colorlinks   = true,
  urlcolor     = blue,
  linkcolor    = blue,
  citecolor    = blue
}

\usepackage[toc, page, header]{appendix}
\begin{document}

\title{Ranking with Abstention}

\author{\name Anqi Mao \email aqmao@cims.nyu.edu \\
       \addr Courant Institute of Mathematical Sciences, New York%
       \AND
       \name Mehryar Mohri \email mohri@google.com \\
       \addr Google Research and Courant Institute of Mathematical Sciences, New York%
       \AND
       \name Yutao Zhong \email yutao@cims.nyu.edu \\
       \addr Courant Institute of Mathematical Sciences, New York%
       }

\editor{}

\maketitle

\begin{abstract}
We introduce a novel framework of \emph{ranking with abstention}, where the learner can abstain from making
prediction at some
limited cost $c$. We present a extensive theoretical analysis of this framework including a series of \emph{$\sH$-consistency bounds} for both the
family of linear functions and that of neural networks with one
hidden-layer. These theoretical guarantees are the state-of-the-art consistency guarantees in the literature, which are upper bounds on the target loss
estimation error of a predictor in a hypothesis set $\sH$, expressed
in terms of the surrogate loss estimation error of that predictor. We further argue that our proposed abstention
methods are important when using common equicontinuous hypothesis sets
in practice. We report the results of experiments illustrating the effectiveness of
ranking with abstention.
\end{abstract}

\section{Introduction}
\label{sec:intro}

In many applications, ranking is a more appropriate formulation of the
learning task than classification, given the crucial significance of
the ordering of the items.  As an example, for
movie recommendation systems, an ordered list of movies is preferable
to a comprehensive list of recommended titles, since users
are more likely to watch those ranked highest.

The problem of learning to rank has been studied in a large number of
publications. The work by \citet{ailon2008efficient,AilonMohri2010} distinguishes
two general formulations of the problem: the \emph{score-based
setting} and the \emph{preference-based setting}. In the score-based
setting, a real-valued function over the input space is learned, whose
values determine a total ordering of all input points. In the
preference-based setting, a pairwise preference function is first
learned, typically by training a classifier over a sample of labeled
pairs; next, that function is used to derive an ordering, potentially
randomized, of any subset of points.

This paper deals with the score-based ranking formulation both in the
general ranking setting, where items are not assigned any specific
category, and the bipartite setting, where they are labeled with one
of two classes. The evaluation of a ranking solution in this context
is based on the average pairwise misranking metric. In the bipartite
setting, this metric is directly related to the AUC (Area Under the
ROC Curve), which coincides with the average correct pairwise ranking
\citep{Hanley1984,cortes2003auc}, also known as the
Wilcoxon-Mann-Whitney statistic.

For most hypothesis sets, directly optimizing the pairwise misranking
loss is intractable. Instead, ranking algorithms resort to a surrogate
loss. As an example, the surrogate loss for RankBoost
\citep{freund2003efficient,rudin2005margin} is based on the
exponential function and that of SVM ranking
\citep{joachims2002optimizing} on the hinge loss. But, what guarantees
can we rely on when minimizing a surrogate loss instead of the
original pairwise misranking loss?

The property often invoked in this context is \emph{Bayes
consistency}, which has been extensively studied for classification
\citep{Zhang2003,bartlett2006convexity,tewari2007consistency}. The
Bayes consistency of ranking surrogate losses has been studied in the
special case of bipartite ranking: in particular,
\citet{uematsu2017theoretically} proved the inconsistency of the
pairwise ranking loss based on the hinge loss and
\citet{gao2015consistency} gave excess loss bounds for pairwise
ranking losses based on the exponential or the logistic loss (see also \citep{menon2014bayes}).  A
related but distinct consistency question has been studied in several
publications \citep{agarwal2005generalization,kotlowski2011bipartite,agarwal2014surrogate}. It is
one with respect to binary classification, that is whether a near
minimizer of the surrogate loss of the binary classification loss is a
near minimizer of the bipartite misranking loss \citep{cortes2003auc}.

However, as recently argued by \citet*{awasthi2022Hconsistency}, Bayes
consistency is not a sufficiently informative notion since it only
applies to the entire class of measurable functions and does not hold
for specific subsets, such as sub-families of linear functions or
neural networks. Furthermore, Bayes consistency is solely an
asymptotic concept and does not offer insights into the performance of
predictors trained on finite samples.  In response, the authors
proposed an alternative concept called \emph{$\sH$-consistency
bounds}, which provide non-asymptotic guarantees tailored to a given
hypothesis set $\sH$. They proceeded to establish such bounds within
the context of classification both in binary and multi-class
classification \citep{awasthi2022Hconsistency,awasthi2022multi}, see
also \citep{mao2023cross,MaoMohriZhong2023ranking,zheng2023revisiting}. These are stronger and more informative
guarantees than Bayes consistency.

But, can we derive $\sH$-consistency bounds guarantees for ranking? We propose a novel framework of \emph{ranking with abstention}, where the learner can abstain from making
prediction at some
limited cost $c$, in both the general pairwise ranking
scenario and
the bipartite ranking scenarios. For surrogate losses of these abstention loss
functions, we give a series of $\sH$-consistency bounds for both the
family of linear functions and that of neural networks with one
hidden-layer. A key term appearing in these bounds is the
\emph{minimizability gap}, which measures the difference between the
best-in-class expected loss and the expected infimum of the pointwise
expected loss.  This plays a crucial role in these bounds and we give
a detailed analysis of these terms.

We will further show that, without abstention, deriving non-trivial
$\sH$-consistency bounds is not
possible for most hypothesis sets used in practice, including the
family of constrained linear models or that of the constrained neural
networks, or any family of equicontinuous functions with respect to
the input. In fact, we will give a relatively simple example where
the pairwise misranking error of the RankBoost algorithm remains
significant, even after training with relatively large sample
sizes. These results further imply the importance of our proposed abstention methods.

We also present the results of
experiments illustrating the effectiveness of ranking with abstention.

\textbf{Technical novelty.} The primary technical differences and challenges between the ranking and
classification settings \citep{awasthi2022Hconsistency} stem from the
fundamental distinction that ranking loss functions take as argument a
pair of samples rather than a single one, as is the case for binary
classification loss functions. This makes it more challenging to
derive $\sH$-consistency bounds, as upper bounding the calibration gap of the target loss by that of the surrogate loss becomes technically more
difficult.

Additionally, this fundamental difference leads to a negative result
for ranking, as $\sH$-consistency bounds cannot be guaranteed for most
commonly used hypothesis sets, including the family of constrained
linear models and that of constrained neural networks, both of which
satisfy the equicontinuity property concerning the input. As a result,
a natural alternative involves using ranking with abstention, for
which $\sH$-consistency bounds can be proven. In the abstention
setting, an extra challenge lies in carefully monitoring the effect of
a threshold $\gamma$ to relate the calibration gap of the target loss
to that of the surrogate loss. 

Furthermore, the bipartite ranking setting introduces an added layer
of complexity, as each element of a pair of samples has an independent
conditional distribution, which results in a more intricate
calibration gap.

\textbf{Structure of the paper.} The remaining sections of this paper
are organized as follows.  In
Section~\ref{sec:general-abs}, we study general pairwise ranking with
abstention. We provide a series of explicit $\sH$-consistency bounds
in the case of the pairwise abstention loss, with multiple choices of
the surrogate loss and for both the family of linear functions and that of neural networks with
one hidden-layer. We also study 
bipartite ranking with abstention in
Section~\ref{sec:bi-abs}. Here too, we present
$\sH$-consistency bounds for bipartite abstention loss, for linear hypothesis sets and the family of neural networks with
one hidden-layer. In Section~\ref{sec:importance}, we show the importance of our abstention methods by demonstrating that without abstention, there exists no meaningful
$\sH$-consistency bound for general surrogate loss functions with an
equicontinuous hypothesis set $\sH$, in both the general pairwise ranking (Section~\ref{sec:neg-general}) and the bipartite ranking (Section~\ref{sec:neg-bi}) scenarios.
In
Section~\ref{sec:experiments}, we report the results of experiments
illustrating the effectiveness of ranking with abstention.

We give a detailed discussion of related work in
Appendix~\ref{app:related-work}.

\section{General Pairwise Ranking with Abstention}
\label{sec:general-abs}
In this section, we introduce a novel framework of
\emph{general pairwise ranking with abstention}. We begin by introducing the
necessary definitions and concepts.

\subsection{Preliminaries}
\label{sec:pre1}
We study the learning scenario of score-based ranking in the
\emph{general pairwise ranking} scenario (e.g.\ see
\citep{MohriRostamizadehTalwalkar2018}).
Let $\sX$ denote the input space and $\sY = \curl*{-1, +1}$ the label
space. We denote by $\sH$ a hypothesis set of functions mapping from
$\sX$ to $\Rset$. The \emph{general pairwise misranking loss} $\lrank$
is defined for all $h$ in $\sH$, $x, x'$ in $\sX$ and $y$ in
$\sY$ by
\begin{align}
\label{eq:pm-loss}
\lrank(h, x, x', y) =  \mathds{1}_{y\neq \sign(h(x') - h(x))},
\end{align}
where $\sign(u) = \mathds{1}_{u \geq 0} - \mathds{1}_{u < 0}$. Thus,
$h$ incurs a loss of one on the labeled pair $(x, x', y)$ when it
ranks the pair $(x, x')$ opposite to the sign of $y$, where, by
convention, $x'$ is considered as ranked above $x$ when $h(x') \geq
h(x)$. Otherwise, the loss incurred is zero.

The framework we propose is that of
\emph{general pairwise ranking with abstention}.
In this framework, the learner abstains from making a prediction on
input pair $(x, x')$ if the distance between $x'$ and $x$ is relatively
small, in which case a cost $c$ is incurred. Let $\norm{\, \cdot \, }$
denote the norm adopted, which is typically an $\ell_p$-norm, $p  \in
[1, + \infty]$. The \emph{pairwise abstention loss} is defined as
follows for any $h \in \sH$ and $(x, x', y)  \in \sX\times\sX \times
\sY$:
\begin{equation}
\label{eq:abs-general}
\labs(h, x, x', y)
=  \1_{y\neq \sign(h(x') - h(x))} \1_{\norm*{x - x'}  >  \gamma}
+  c\, \1_{\norm*{x - x'} \leq \gamma},
\end{equation}
where $\gamma$ is a given threshold value. For $\gamma = 0$, $\labs$
reduces to the pairwise misranking loss $\lrank$ without
abstention.

In Section~\ref{sec:importance}, we will show the importance of our proposed abstention
methods when using common equicontinuous hypothesis sets
in practice. Optimizing the pairwise misranking loss $\lrank$ or pairwise abstention loss $\labs$ is intractable for
most hypothesis sets. Thus, general ranking algorithms rely on a
surrogate loss function $\sfL$ instead of $\lrank$. 
The general pairwise ranking surrogate losses widely used in practice
admit the following form:
\begin{align}
\label{eq:sur-loss}
\sfL_{\Phi}(h, x, x', y)  =  \Phi\paren[\big]{y(h(x') - h(x))},
\end{align}
where $\Phi$ is a non-increasing function that is continuous at $0$
and upper bounding $u \mapsto \mathds{1}_{u\leq 0}$ over $\Rset$. 
We will analyze
the properties of such surrogate loss functions with respect to both $\lrank$ and $\labs$. We will specifically consider the hinge loss $\Phi_{\mathrm{hinge}}(t) = \max\curl*{0, 1 - t}$, the exponential loss $\Phi_{\mathrm{exp}}(t) = e^{-t}$ and the sigmoid loss $\Phi_{\mathrm{sig}}(t) = 1 - \tanh(kt),~k > 0$ as auxiliary functions $\Phi$.

Let $\sD$ denote a distribution over $\sX \times \sX \times \sY$.  We
denote by $\eta(x, x') = \sD(Y = \plus 1 \!\mid\! (X, X') = (x, x'))$
the conditional probability of $Y = \plus 1$ given $(X, X') = (x,
x')$.  We also denote by $\sR_{\sfL}(h)$ the \emph{expected
$\sfL$-loss} of a hypothesis $h$ and by $\sR_{\sfL}^*(\sH)$ its
infimum over $\sH$:
\[
\sR_{\sfL}(h)
= \E_{(x, x', y) \sim \sD}\bracket*{\sfL(h, x, x', y)}
\quad
\sR_{\sfL}^*(\sH) = \inf_{h \in\sH}\sR_{\sfL}(h)
\]
\textbf{$\sH$-consistency bounds.} We will analyze the
\emph{$\sH$-consistency bounds} properties
\citep{awasthi2022Hconsistency} of such surrogate loss functions. An
$\sH$-consistency bound for a surrogate loss $\sfL$ and a target loss $\ov \sfL$ is a guarantee of
the form:
\[
\forall h \in \sH, \quad \sR_{\ov \sfL}(h) -
\sR_{\ov \sfL}^*(\sH) \leq f\paren*{\sR_{\sfL}(h) -  \sR^*_{\sfL}(\sH)},
\]
for some non-decreasing function $f\colon \Rset_{+}\to
\Rset_{+}$, where $\ov \sfL$ can be taken as $\lrank$ or $\labs$. This provides a quantitative relationship between the
estimation loss of $\ov \sfL$ and that of the surrogate loss $\sfL$.
The guarantee is stronger and more informative than Bayes consistency,
or $\sH$-consistency, $\sH$-calibration or the excess error bounds
\citep{Zhang2003,bartlett2006convexity,steinwart2007compare,
  MohriRostamizadehTalwalkar2018} discussed in the literature.

A key quantity appearing in $\sH$-consistency bounds is the
\emph{minimizability gap}, which is the difference between the
best-in-class expected loss and the expected pointwise infimum of the
loss:
\[
\sM_{\sfL}(\sH) = \sR^*_{\sfL}(\sH) - \E_{(x, x')}
\bracket*{\inf_{h \in \sH} \E_{y}\bracket*{\sfL(h, x, x', y)
    \mid (x, x')}}.
\]
By the super-additivity of the infimum, the minimizability gap is
always non-negative.

We will specifically study the hypothesis set of linear hypotheses,
$\sH_{\mathrm{lin}} = \big\{x \mapsto w \cdot x + b \mid \norm*{w}_q
  \leq W, \abs*{b} \leq B\big\}$ and the hypothesis set of
one-hidden-layer ReLU networks: $\sH_{\mathrm{NN}} = \big\{x \mapsto
\sum_{j = 1}^n u_j(w_j \cdot x + b_j)_{+} \mid \|u
\|_{1}\leq\Lambda,\norm{w_j}_q \leq W, \abs*{b_j}\leq B\big\}$, where
$(\cdot)_+ = \max(\cdot,0)$.

Let $p, q \in[1, \plus \infty]$ be conjugate numbers, that
is $\frac{1}{p} + \frac{1}{q} = 1$. Without loss of generality, we
consider $\sX = B_p^d(1)$ and $\norm*{\cdot}$  in \eqref{eq:abs-general} to be the 
$\ell_p$ norm. The corresponding conjugate $\ell_q$ norm
is adopted in the hypothesis sets $\sH_{\mathrm{lin}}$ and
$\sH_{\mathrm{NN}}$. In the following, we will prove $\sH$-consistency
bounds for $\sfL=\sfL_{\Phi}$ and $\ov\sfL=\labs$ when using as an auxiliary function $\Phi$
the hinge loss, the exponential loss, or the sigmoid loss, in the case of the linear hypothesis set
$\sH_{\mathrm{lin}}$ or that of one-hidden-layer ReLU
networks $\sH_{\mathrm{NN}}$.

\subsection{$\sH$-consistency bounds for pairwise abstention loss}
\label{sec:general-bound}
Theorem~\ref{thm:general-bound} shows the $\sH$-consistency
bounds for $\sfL_{\Phi}$ with respect to $\labs$ when using common auxiliary functions. The bounds in
Theorem~\ref{thm:general-bound} depend directly on the threshold value
$\gamma$, the parameter $W$ in the linear models and parameters of the
loss function (e.g., $k$ in sigmoid loss). Different from the bounds in the
linear case, all the bounds for one-hidden-layer ReLU networks not only depend on
$W$, but also depend on $\Lambda$, which is a parameter appearing in
$\sH_{\mathrm{NN}}$. 
\begin{restatable}[\textbf{$\sH$-consistency bounds for pairwise abstention loss}]{theorem}{GeneralBound}
\label{thm:general-bound}
Let $\sH$ be $\sH_{\mathrm{lin}}$ or $\sH_{\mathrm{NN}}$. Then, for any $h\in \sH$ and any distribution,
\ifdim\columnwidth=\textwidth
{
\begin{equation*}
    \sR_{\labs}(h) - \sR_{\labs}^*(\sH) + \sM_{\labs}(\sH)\leq 
    \Gamma_{\Phi}\paren*{\sR_{\sfL_{\Phi}}(h) -  \sR_{\sfL_{\Phi}}^*(\sH) + \sM_{\sfL_{\Phi}}(\sH)},
\end{equation*}
}
\else
{
\begin{multline*}
\sR_{\labs}(h) - \sR_{\labs}^*(\sH) + \sM_{\labs}(\sH)\\
\leq 
    \Gamma_{\Phi}\paren*{\sR_{\sfL_{\Phi}}(h) -  \sR_{\sfL_{\Phi}}^*(\sH) + \sM_{\sfL_{\Phi}}(\sH)},
\end{multline*}
}
\fi
where $\Gamma_{\Phi}(t)=\frac{t}{\min\curl*{W\gamma,1}}$, $\max\curl*{\sqrt{2t},2\paren*{\frac{e^{2W\gamma} + 1}{e^{2W\gamma} - 1}}\, t}$ and $\frac{t}{\tanh\paren*{kW\gamma}}$ for $\Phi=\Phi_{\mathrm{hinge}}$, $\Phi_{\mathrm{exp}}$ and $\Phi_{\mathrm{sig}}$ respectively. $W$ is replaced by $\Lambda W$ for
$\sH=\sH_{\mathrm{NN}}$.
\end{restatable}

As an example, for $\sH=\sH_{\rm{lin}}$ or $\sH_{\rm{NN}}$, when using as $\Phi$ the exponential loss function,
modulo the minimizability gaps (which are zero when the best-in-class
error coincides with the Bayes error or can be small in some other
cases), the bound implies that if the surrogate estimation loss
$\sR_{\sfL_{\Phi_{\mathrm{exp}}}}(h) - 
\sR_{\sfL_{\Phi_{\mathrm{exp}}}}^*(\sH)$ is reduced 
to $\e$, then, the target estimation loss $\sR_{\labs}(h) - 
\sR_{\labs}^*(\sH)$ is upper bounded by
$\Gamma_{\Phi_{\mathrm{exp}}}(\e)$.  For sufficiently small values of $\e$, the
dependence of $\Gamma_{\Phi_{\mathrm{exp}}}$ on $\e$ exhibits a square root
relationship. However, if this is not the case, the dependence becomes
linear, subject to a constant factor depending on the threshold value
$\gamma$, the parameter $W$ in the linear models and the one-hidden-layer ReLU networks, and an additional parameter $\Lambda$ in
the one-hidden-layer ReLU networks.

The proofs consist of analyzing calibration gaps of the target
loss and that of each surrogate loss and seeking a tight lower bound of
the surrogate calibration gap in terms of the target one. As an example, for
$\Phi = \Phi_{\mathrm{exp}}$, we have the tight lower bound
$\Delta\sC_{\sfL_{\Phi_{\mathrm{exp}}},\sH}(h, x, x')\geq
\Delta\sC_{\sfL_{\Phi_{\mathrm{exp}}},\sH}(h_0,x, x') = 
\Psi_{\rm{exp}}
\paren*{\Delta\sC_{\labs,\sH}(h, x, x')}$, where $h_0$
can be the null hypothesis when
$\Delta\sC_{\labs,\sH}(h, x, x')\neq 0$ and
$\Psi_{\rm{exp}}$ is an increasing and piecewise convex function on
$[0, 1]$ defined by \[\Psi_{\rm{exp}}(t) = \begin{cases} 1 - \sqrt{1 - t^2}, &
t\leq
\frac{e^{2W\gamma} - 1}{e^{2W\gamma} + 1}\\ 1 - \frac{t + 1}{2}e^{-W\gamma} - \frac{1 - t}{2}e^{W\gamma},
& t >  \frac{e^{2W\gamma} - 1}{e^{2W\gamma} + 1}
\end{cases},\] where $W$ is replaced by $\Lambda W$ for
$\sH=\sH_{\mathrm{NN}}$. The detailed proofs and the expression of the corresponding minimizability gaps are included in Appendix~\ref{app:abstention-general}.

\section{Bipartite Ranking with Abstention}
\label{sec:bi-abs}

\begin{table*}[t]
\vskip -0.1in
\caption{General pairwise abstention loss for the Rankboost loss on CIFAR-10;
  mean $\pm$ standard deviation over three runs for various $\gamma$
  and cost $c$.}
\vskip -0.2in
    \label{tab:comparison-ranking}
\begin{center}
\resizebox{1.0\textwidth}{!}{
    \begin{tabular}{@{\hspace{0pt}}llllll@{\hspace{0pt}}}
      $\gamma$ & $0$ & $0.3$ & $0.5$ & $0.7$ & $0.9$ \\
      \toprule
     Cost $0.1$  & 8.33\% $\pm$ 0.15\%  & 8.33\% $\pm$ 0.15\% &  8.33\% $\pm$ 0.15\% &  8.25\% $\pm$ 0.07\% & 8.54\%$\pm$ 0.07\% \\
      \midrule
     Cost $0.3$ & 8.33\% $\pm$ 0.15\%  & 8.33\% $\pm$ 0.15\% &  8.35\% $\pm$ 0.15\% &  9.73\% $\pm$ 0.11\% & 20.41\%$\pm$ 0.06\% \\
     \midrule
     Cost $0.5$ &  8.33\% $\pm$ 0.15\%  & 8.33\% $\pm$ 0.15\% &  8.36\% $\pm$ 0.14\% &  11.20\% $\pm$ 0.14\% & 32.28\% $\pm$ 0.07\% \\
    \end{tabular}
    }
\end{center}
    \vskip -0.2in
\end{table*}

As with the general pairwise ranking case, we introduce a novel framework of
\emph{bipartite
ranking with abstention}. We first introduce the relevant
definitions and concepts.
\subsection{Preliminaries}
\label{sec:pre2}

In the bipartite setting, each point $x$ admits a label $y \in
\curl*{-1, +1}$.  The \emph{bipartite misranking loss} $\lrankbi$ is
defined for all $h$ in $\sH$, and $(x, y), (x', y')$ in $(\sX \times \sY)$ by
\begin{equation}
\label{eq:bm-loss}
\lrankbi(h, x, x', y, y') =  \\
\mathds{1}_{(y - y')(h(x) - h(x')) < 0}+
\frac{1}{2} \mathds{1}_{(h(x) = h(x'))\wedge (y\neq y')}.
\end{equation}
The framework we propose is that of \emph{bipartite ranking with abstention}. In this framework, the learner can abstain from making
prediction on a pair $(x, x')$ with $x$ and $x'$ relatively close. The
\emph{bipartite abstention loss} is defined as follows for any $h \in
\sH$ and $(x, y), (x', y')  \in \sX\times \sY$:
\begin{equation}
\label{eq:abs-bi}
\labsbi(h, x, x', y, y')\\
 =  \lrankbi(h, x, x', y, y') \1_{\norm*{x - x'}  >  \gamma} +  c\, \1_{\norm*{x - x'} \leq \gamma},
\end{equation}
where $\gamma$ is a given threshold value. When $\gamma = 0$, $\labsbi$
reduces to bipartite misranking loss $\lrankbi$ without abstention.

Optimizing the bipartite misranking loss $\lrankbi$ or bipartite abstention loss $\labsbi$ is intractable for
most hypothesis sets and bipartite ranking algorithms rely instead on
a surrogate loss $\lbi$. The bipartite ranking surrogate losses widely used in practice, admit
the following form:
\begin{equation}
\label{eq:sur-loss-bi}
\mspace{-1mu}  
\lbi_{\Phi}(h, x, x', y, y')
\mspace{-4mu}  
=
\mspace{-4mu}  
\Phi\paren*{\frac{(y - y') \paren*{h(x) - h(x')}}{2}}
\mathds{1}_{y \neq y'},
\mspace{-9mu}  
\end{equation}
where $\Phi$ is a non-increasing function that is continuous at $0$
upper bounding $u \mapsto \mathds{1}_{u\leq 0}$ over $\Rset$. We will analyze the
\emph{$\sH$-consistency bounds} properties
\citep{awasthi2022Hconsistency} of such surrogate loss functions with respect to both $\lrankbi$ and $\labsbi$. As with the general pairwise ranking case, we will specifically consider the hinge loss $\Phi_{\mathrm{hinge}}(t) = \max\curl*{0, 1 - t}$, the exponential loss $\Phi_{\mathrm{exp}}(t) = e^{-t}$ and the sigmoid loss $\Phi_{\mathrm{sig}}(t) = 1 - \tanh(kt),~k > 0$ as auxiliary functions $\Phi$.

Let $\sD$ be a distribution over $\sX \times \sY$.
We denote by $\eta(x) = \sD(Y = \plus 1 \!\mid\! X = x)$ the
conditional probability of $Y = \plus 1$ given $X = x$.
We will use a definition and notation for the expected $\lbi$-loss of
$h \in \sH$, its infimum, and the minimizability gaps similar to what
we used in the general pairwise misranking setting:
\begin{align*}
 \sR_{\lbi}(h)
& = \E_{(x, x', y) \sim \sD}\bracket*{\lbi(h, x, x', y)}
\quad
\sR_{\lbi}^*(\sH) = \inf_{h \in\sH}\sR_{\lbi}(h)\\
\sM_{\lbi}(\sH)
& = \sR^*_{\lbi}(\sH) - \E_{(x, x')}
\bracket*{\inf_{h \in \sH} \E_{(y, y')}\bracket*{\sfL(h, x,
    x', y, y') \mid (x, x')}}.
\end{align*}

\subsection{$\sH$-consistency bounds for bipartite abstention losses}
\label{sec:bi-bound}
Theorem~\ref{thm:bi-bound} presents a series of
$\sH$-consistency bounds for $\lbi_{\Phi}$  when using as an auxiliary function $\Phi$
the hinge loss, the exponential loss, or the sigmoid loss. The bounds in
Theorem~\ref{thm:bi-bound} depend directly on the threshold value
$\gamma$, the parameter $W$ in the linear models and parameters of the
loss function (e.g., $k$ in sigmoid loss). Different from the bounds in the
linear case, all the bounds for one-hidden-layer ReLU networks not only depend
on $W$, but also depend on $\Lambda$, a  parameter in
$\sH_{\mathrm{NN}}$.
\begin{restatable}[\textbf{$\sH$-consistency bounds for bipartite abstention losses}]{theorem}{BiBound}
\label{thm:bi-bound}
Let $\sH$ be $\sH_{\mathrm{lin}}$ or $\sH_{\mathrm{NN}}$. Then, for any $h\in \sH$ and any distribution,
\ifdim\columnwidth=\textwidth
{
\begin{equation*}
\sR_{\labsbi}(h) - \sR_{\labsbi}^*(\sH) + \sM_{\labsbi}(\sH)
\leq 
    \Gamma_{\Phi}\paren*{\sR_{\lbi_{\Phi}}(h) -  \sR_{\lbi_{\Phi}}^*(\sH) + \sM_{\lbi_{\Phi}}(\sH)}
\end{equation*}
}
\else
{
\begin{multline*}
\sR_{\labsbi}(h) - \sR_{\labsbi}^*(\sH) + \sM_{\labsbi}(\sH)\\
\leq 
    \Gamma_{\Phi}\paren*{\sR_{\lbi_{\Phi}}(h) -  \sR_{\lbi_{\Phi}}^*(\sH) + \sM_{\lbi_{\Phi}}(\sH)}
\end{multline*}
}
\fi
where $\Gamma_{\Phi}(t)$ equals $\frac{t}{\min\curl*{W\gamma,1}}$, $\max\curl*{\sqrt{t},\paren*{\frac{e^{2W\gamma} + 1}{e^{2W\gamma} - 1}}\, t}$ and $\frac{t}{\tanh\paren*{kW\gamma}}$ for $\Phi$ equals
$\Phi_{\mathrm{hinge}}$, $\Phi_{\mathrm{exp}}$ and $\Phi_{\mathrm{sig}}$ respectively. $W$ is replaced by $\Lambda W$ for
$\sH=\sH_{\mathrm{NN}}$.
\end{restatable}

As an example, for $\sH=\sH_{\rm{lin}}$ or $\sH_{\rm{NN}}$, when adopting the exponential loss function as $\Phi$,
modulo the minimizability gaps (which are zero when the best-in-class
error coincides with the Bayes error or can be small in some other
cases), the bound implies that if the surrogate estimation loss
$\sR_{\lbi_{\Phi_{\mathrm{exp}}}}(h) - 
\sR_{\lbi_{\Phi_{\mathrm{exp}}}}^*(\sH)$ is reduced 
to $\e$, then, the target estimation loss $\sR_{\labsbi}(h) - 
\sR_{\labsbi}^*(\sH)$ is upper bounded by
$\Gamma_{\Phi_{\mathrm{exp}}}(\e)$.  For sufficiently small values of $\e$, the
dependence of $\Gamma_{\Phi_{\mathrm{exp}}}$ on $\e$ exhibits a square root
relationship. However, if this is not the case, the dependence becomes
linear, subject to a constant factor depending on the threshold value
$\gamma$, the parameter $W$ in the linear models and the one-hidden-layer ReLU networks, and an additional parameter $\Lambda$ in
the one-hidden-layer ReLU networks.

As with the general pairwise ranking setting, the proofs consist of
analyzing calibration gaps of the target loss and that of each surrogate
loss and seeking a tight lower bound of the surrogate calibration gap in terms
of the target one. Additionally, the bipartite ranking setting
introduces an added layer of complexity, as  $x$ and $x'$ in a
pair have independent conditional distributions $\eta(x)$ and
$\eta(x')$, which results in a more intricate calibration gap that is
harder to address.

As an example, for $\Phi = \Phi_{\mathrm{exp}}$ the exponential loss
function, we have the lower bound
\[\Delta\sC_{\lbi_{\Phi_{\mathrm{exp}}},\sH}(h, x, x')\geq
\Psi_{\rm{exp}}
\paren*{\Delta\sC_{\labs,\sH}(h, x, x')},\] where
$\Psi_{\rm{exp}}$ is an increasing and piece-wise convex function on
$[0,2]$ defined by
\[\Psi_{\rm{exp}}(t) = \min\curl*{t^2,\paren*{\frac{e^{2W\gamma} + 1}{e^{2W\gamma} - 1}}\,
  t},\] where $W$ is replaced by $\Lambda W$ for
$\sH=\sH_{\mathrm{NN}}$. The detailed proofs and the expression of the corresponding
minimizability gaps are included in
Appendix~\ref{app:abstention-bi}.

\section{Importance of abstention}
\label{sec:importance}
In this section, we show the importance of our abstention methods by
demonstrating the impossibility of deriving non-trivial
$\sH$-consistency bounds with respect to $\lrank$ or $\lrankbi$ for widely used surrogate losses and
hypothesis sets. 

\subsection{Negative Results for General Pairwise Ranking}
\label{sec:neg-general}

Here, we give a negative result for standard general pairwise ranking. We will say that a hypothesis set is \emph{regular for general
pairwise ranking} if, for any $x \neq x' \in \sX$, we have
$\curl[big]{\sign(h(x') - h(x))\colon h \in \sH} = \curl*{-1,
  +1}$. Hypothesis sets commonly used in practice all admit
this property.

The
following result shows that the common surrogate losses do not benefit from
a non-trivial $\sH$-consistency bound when the hypothesis set used is
equicontinuous, which includes most hypothesis sets used in practice,
in particular the family of linear hypotheses and that of neural
networks.

\begin{restatable}[\textbf{Negative results}]{theorem}{Negative}
\label{Thm:negative-general}
Assume that $\sX$ contains an interior point $x_0$ and that $\sH$ is
regular for general pairwise ranking, contains $0$ and is
equicontinuous at $x_0$. If for
some function $f$ that is non-decreasing and continuous at $0$, the
following bound holds for all $h \in \sH$ and any distribution,
\begin{align*}
\sR_{\lrank}(h) - \sR_{\lrank}^*(\sH)
    \leq f\paren*{\sR_{\sfL_{\Phi}}(h) - \sR_{\sfL_{\Phi}}^*(\sH)},
\end{align*}
then, $f(t)\geq 1$ for any $t\geq 0$.
\end{restatable}
Theorem~\ref{Thm:negative-general} shows that for equicontinuous
hypothesis sets, any $\sH$-consistency bound is vacuous, assuming that
$f$ is a non-decreasing function continuous at zero. This is because
for any such bound, a small $\sfL_{\Phi}$-estimation loss does not
guarantee a small $\lrank$-estimation loss, as the right-hand side
remains lower-bounded by one.

The proof is given in Appendix~\ref{app:general-negative}, where we
give a simple example on pairs whose distance is relatively small for
which the standard surrogate losses including the RankBoost algorithm
($\sfL_{\rm{exp}}$) fail (see also Section~\ref{sec:experiments}). It
is straightforward to see that the assumptions of
Theorem~\ref{Thm:negative-general} hold for the case $\sH =
\sH_{\mathrm{lin}}$ or $\sH = \sH_{\mathrm{NN}}$. Indeed, we can take
$x_0 = 0$ as the interior point and thus for any $h \in
\sH_{\mathrm{lin}}$, $\abs*{h(x) - h(x_0)} = \abs*{w\cdot x} < \e$ for
any $x \in \curl*{x \in \sX:\norm*{x}_p < \frac{\e}{W}}$, which
implies that $\sH_{\mathrm{lin}}$ is equicontinuous at $x_0$. As with
the linear hypothesis set, for any $h \in \sH_{\mathrm{NN}}$,
$\abs*{h(x) - h(x_0)} = \abs*{\sum_{j = 1}^n u_j(w_j \cdot x + b_j)_{+}
  - \sum_{j = 1}^n u_j(b_j)_{+}} = \abs*{\sum_{j = 1}^n
  u_j\bracket*{(w_j \cdot x + b_j)_{+} - (b_j)_{+}}}\leq \Lambda
W\norm*{x}_p < \e,$
for any $x \in \curl*{x \in \sX:\norm*{x}_p <  \frac{\e}{\Lambda W}}$,
which implies that $\sH_{\mathrm{NN}}$ is equicontinuous at $x_0$.  In
fact, Theorem~\ref{Thm:negative-general} holds for any family of
Lipschitz constrained neural networks, since a family of functions
that share the same Lipschitz constant is equicontinuous.

It is straightforward to verify that the proof of
Theorem~\ref{Thm:negative-general} also holds in the deterministic
case where $\eta(x, x')$ equals $0$ or $1$ for any $x \neq x'$, which
yields the following corollary.

\begin{restatable}[\textbf{Negative results in the deterministic case}]{corollary}{Negative-deterministic}
\label{cor:negative}
In the deterministic case where $\eta(x, x')$ equals $0$ or $1$ for any
$x \neq x'$, the negative result of Theorem~\ref{Thm:negative-general}
still holds.
\end{restatable}

\subsection{Negative Results for Bipartite Ranking}
\label{sec:neg-bi}

Here, as in the general pairwise misranking scenario, we present a
negative result in the standard bipartite setting. We say that a hypothesis set is \emph{regular for bipartite ranking}
if, for any $x \neq x' \in \sX$, there exists $h_{ + } \in \sH$ such
that $h_{ + }(x) < h_{ + }(x')$ and $h_{-} \in \sH$ such that
$h_{-}(x) > h_{-}(x')$. Hypothesis sets commonly used in practice all
admit this property.

As with the general pairwise ranking, we show that common surrogate
losses do not benefit from $\sH$-consistency bounds when $\sH$ is an
equicontinuous family.

\begin{restatable}[\textbf{Negative results
      for bipartite ranking}]{theorem}{NegativeBi}
\label{Thm:negative-bi}
Assume that $\sX$ contains an interior point $x_0$ and that $\sH$ is
regular for bipartite ranking, contains $0$ and is equicontinuous at
$x_0$. If for some function $f$ that
is non-decreasing and continuous at $0$, the following bound holds for
all $h \in \sH$ and any distribution,
\begin{align*}
\sR_{\lrankbi}(h) - \sR_{\lrankbi}^*(\sH)
    \leq f\paren*{\sR_{\lbi_{\Phi}}(h) - \sR_{\lbi_{\Phi}}^*(\sH)},
\end{align*}
then, $f(t)\geq \frac{1}{2}$ for any $t\geq 0$.
\end{restatable}
As with Theorem~\ref{Thm:negative-general},
Theorem~\ref{Thm:negative-bi} shows that in the bipartite ranking
setting, any $\sH$-consistency bound with an equicontinuous hypothesis
set is vacuous, assuming a non-decreasing function $f$ continuous at
zero. The proof is given in Appendix~\ref{app:negative-bi}. It is
straightforward to verify that the proof holds in the deterministic
case where $\eta(x)$ equals $0$ or $1$ for any $x \in \sX$, which
yields the following corollary.

\begin{restatable}[\textbf{Negative results in the
      bipartite deterministic
      case}]{corollary}{Negative-deterministic-bi}
\label{cor:negative-bi}
In the bipartite deterministic case where $\eta(x)$ equals $0$ or $1$
for any $x \in \sX$, the same negative result as in
Theorem~\ref{Thm:negative-bi} holds.
\end{restatable}
The negative results in Section~\ref{sec:neg-general} and Section~\ref{sec:neg-bi} suggest that without abstention, standard pairwise ranking with theoretical guarantees is difficult with common
hypothesis sets. The inherent issue for pairwise ranking is that for
equicontinuous hypotheses, when $x$ and $x'$ are arbitrarily close,
the confidence value $\abs*{h(x) - h(x')}$ can be arbitrary close to
zero. These results further imply the importance of ranking with abstention, where the learner can abstain from making
prediction on a pair $(x, x')$ with $x$ and $x'$ relatively close, as illustrated in Section~\ref{sec:general-abs} and Section~\ref{sec:bi-abs}.

\section{Experiments}
\label{sec:experiments}

In this section, we provide empirical results for general pairwise
ranking with abstention on the CIFAR-10 dataset
\citep{Krizhevsky09learningmultiple}.

We used ResNet-34 with ReLU activations \citep{he2016deep}. Here,
ResNet-$n$ denotes a residual network with $n$ convolutional
layers. Standard data augmentations, 4-pixel padding with $32 \times
32$ random crops and random horizontal flips are applied for
CIFAR-10. For training, we used Stochastic Gradient Descent (SGD) with
Nesterov momentum \citep{nesterov1983method}. We set the batch size,
weight decay, and initial learning rate to $1\mathord,024$, $1\times
10^{-4}$ and $0.1$ respectively. We adopted the cosine decay learning
rate schedule \citep{loshchilov2016sgdr} for a total of $200$
epochs. The pairs $(x, x', y)$ are randomly sampled from CIFAR-10
during training, with $y = \pm 1$ indicating if $x$ is ranked above or
below $x'$ per the natural ordering of labels of $x$ and
$x'$.

We evaluated the models based on their averaged pairwise abstention
loss \eqref{eq:abs-general} with $\gamma$ selected from $\curl*{0.0,
  0.3, 0.5, 0.7, 0.9}$ and the cost $c$ selected from $\curl*{0.1,
  0.3, 0.5}$.  We randomly sampled $10\mathord,000$ pairs $(x, x')$
from the test data for evaluation. The $\ell_{\infty}$ distance is adopted in
the algorithm. We averaged losses over three runs and report the
standard deviation as well.

We used the surrogate loss \eqref{eq:sur-loss} with $\Phi(t) =
\exp(-t)$ the exponential loss, $\sfL_{\Phi_{\rm{exp}}}$, which
coincides with the loss function of
RankBoost. Table~\ref{tab:comparison-ranking} shows that when $\gamma$
is as small as $0.3$, no abstention takes place and the abstention
loss coincides with the standard misranking loss ($\gamma = 0$) for
any cost $c$. As $\gamma$ increases, there are more samples that are
abstained.
When using a minimal cost $c$ of 0.1 (as demonstrated in the first row
of Table~\ref{tab:comparison-ranking}), abstaining on pairs with a
relatively small distance ($\gamma = 0.7$) results in a lower target
abstention loss compared to the scenario without abstention ($\gamma =
0$). Conversely, abstaining on pairs with larger distances ($\gamma =
0.9$) led to a higher abstention loss. This can be attributed to the
fact that rejected samples at $\gamma = 0.7$ had lower accuracy
compared to those at $\gamma = 0.9$. This empirically verifies that the surrogate loss
$\sfL_{\Phi_{\rm{exp}}}$ is not favorable on pairs whose distance is
relatively small, for equicontinuous hypotheses. When the cost $c$ is larger, the
abstention loss, in general, increases with $\gamma$, since the number
of samples rejected increases with $\gamma$.

Overall, the experiment shows that, in practice, for small $\gamma$,
 abstention actually does not take place. Thus, the abstention loss
coincides with the standard pairwise misranking loss in those cases,
and the surrogate loss is consistent with respect to both of them.
Our results also indicate that the surrogate loss
$\sfL_{\Phi_{\rm{exp}}}$, a commonly used loss function, for example
for RankBoost, is not optimal for pairs with a relatively small
distance. Instead, rejecting these pairs at a minimal cost proves to
be a more effective strategy.

\section{Conclusion}

We introduce a novel framework of ranking with abstention, in both the general pairwise ranking and the bipartite ranking scenarios. Our proposed abstention
methods are important when using common equicontinuous hypothesis sets
in practice. It will be useful to explore alternative
non-equicontinuous hypothesis sets that may be of practical use, and
to further study the choice of the parameter $\gamma$ for abstention
in practice. We have also initiated the study of randomized ranking
solutions with theoretical guarantees.

\bibliography{hcbr}

\newpage
\appendix
\onecolumn

\renewcommand{\contentsname}{Contents of Appendix}
\tableofcontents
\addtocontents{toc}{\protect\setcounter{tocdepth}{3}} 
\clearpage

\section{Related work}
\label{app:related-work}

The notions of Bayes consistency (also known as consistency) and
calibration have been extensively studied for classification
\citep{Zhang2003,bartlett2006convexity,tewari2007consistency}. The
Bayes consistency of ranking surrogate losses has been studied in the
special case of bipartite score-based ranking: in particular,
\citet{uematsu2017theoretically} proved the inconsistency of the
pairwise ranking loss based on the hinge loss and
\citet{gao2015consistency} gave excess loss bounds for pairwise
ranking losses based on the exponential or the logistic loss. Later,
these results were further generalized by \citet{menon2014bayes}.  A
related but distinct consistency question has been studied in several
publications \citep{agarwal2005generalization,kotlowski2011bipartite,
  agarwal2014surrogate}. It is one with respect to binary
classification, that is whether a near minimizer of the surrogate loss
of the binary classification loss is a near minimizer of the bipartite
misranking loss \citep{cortes2003auc}.

Considerable attention has been devoted to the study of the learning
to rank algorithms and their related problems: including one-pass AUC
pairwise optimization \citep{gao2013one}, preference-based ranking
\citep{cohen1997learning,clemenccon2008ranking}, subset ranking with
Discounted Cumulative Gain (DCG)
\citep{cossock2008statistical,buffoni2011learning}, listwise ranking
\citep{xia2008listwise}, subset ranking based on Pairwise Disagreement
(PD) \citep{duchi2010consistency,lan2012statistical}, subset ranking
using Normalized Discounted Cumulative Gain (NDCG)
\citep{ravikumar2011ndcg}, subset ranking with Average Precision (AP)
\citep{calauzenes2012non,ramaswamy2013convex}, general multi-class
problems \citep{ramaswamy2012classification,ramaswamy2014consistency}
and multi-label problems
\citep{gaozhou2011consistency,zhang2020convex}.

Bayes consistency only holds for the full family of measurable
functions, which of course is distinct from the more restricted
hypothesis set used by a learning algorithm. Therefore, a hypothesis
set-dependent notion of $\sH$-consistency has been proposed by
\citet{long2013consistency} in the realizable setting, which was used
by \citet{zhang2020bayes} for linear models, and generalized by
\citet{KuznetsovMohriSyed2014} to the structured prediction case.
\citet{long2013consistency} showed that there exists a case where a
Bayes-consistent loss is not $\sH$-consistent while inconsistent loss
functions can be $\sH$-consistent. \citet{zhang2020bayes} further
investigated the phenomenon in \citep{long2013consistency} and showed
that the situation of loss functions that are not $\sH$-consistent
with linear models can be remedied by carefully choosing a larger
piecewise linear hypothesis set.  \citet{KuznetsovMohriSyed2014}
proved positive results for the $\sH$-consistency of several
multi-class ensemble algorithms, as an extension of $\sH$-consistency
results in \citep{long2013consistency}.

Recently, \citet{awasthi2022Hconsistency} presented a series of
results providing $\sH$-consistency bounds in binary
classification. These guarantees are significantly stronger than the
$\sH$-calibration or $\sH$-consistency properties studied by
\citet{awasthi2021calibration,awasthi2021finer}.
\citet{awasthi2022multi} and \citet{mao2023cross} (see also
\citep{zheng2023revisiting}) generalized $\sH$-consistency bounds to
the scenario of multi-class
classification. \citet{awasthi2023theoretically} proposed a family of
loss functions that benefit from such $\sH$-consistency bounds
guarantees for adversarial robustness
\citep{goodfellow2014explaining,madry2017towards,tsipras2018robustness,
  carlini2017towards,awasthi2023dc}. \citet{MaoMohriZhong2023ranking} used $\sH$-consistency bounds in the context of ranking. $\sH$-consistency bounds are also
more informative than similar excess error bounds derived in the
literature, which correspond to the special case where $\sH$ is the
family of all measurable functions
\citep{Zhang2003,bartlett2006convexity,MohriRostamizadehTalwalkar2018}. Our
work significantly generalizes the results of
\citet{awasthi2022Hconsistency} to the score-based ranking setting,
including both the general pairwise ranking and bipartite ranking scenarios.

\section{General tools}
\label{app:proofs_general}

To begin with the proof, we first introduce some notation. In general
pairwise ranking scenario, we denote by $\sD$ a distribution over $\sX
\times \sX \times \sY$ and by $\sP$ a set of such distributions. We
further denote by $\eta(x, x')  =  \sD(Y  =  1 \!\mid\! (X, X')  =  (x, x'))$
the conditional probability of $Y = 1$ given $(X, X')  =  (x, x')$. Without
loss of generality, we assume that $\eta(x, x) = 1/2$. The generalization
error for a surrogate loss $\sfL$ can be rewritten as $ \sR_{\sfL}(h)
 =  \mathbb{E}_{X} \bracket*{\sC_{\sfL}(h, x, x')} $, where
$\sC_{\sfL}(h, x, x')$ is the conditional $\sfL$-risk, defined by
\begin{align*}
\sC_{\sfL}(h, x, x')  = \eta(x, x') \sfL(h, x, x', + 1) + \paren*{1 - \eta(x, x')}\sfL(h, x, x',-1).
\end{align*}
We denote by $\sC_{\sfL}^*(\sH,x, x')  =   \inf_{h \in
  \sH}\sC_{\sfL}(h, x, x')$ the minimal conditional
$\sfL$-risk. Then, the minimizability gap can be rewritten as follows:
\begin{align*}
\sM_{\sfL}(\sH)
  =  \sR^*_{\sfL}(\sH) - \mathbb{E}_{X} \bracket* {\sC_{\sfL}^*(\sH,x)}.
\end{align*}
We further refer to $\sC_{\sfL}(h, x, x') - \sC_{\sfL}^*(\sH,x, x')$ as
the calibration gap and denote it by $\Delta\sC_{\sfL,\sH}(h, x, x')$. 

In bipartite ranking scenario, we denote by $\sD$ a distribution over
$\sX \times \sY$ and by $\sP$ a set of such distributions. We further
denote by $\eta(x)  =  \sD(Y  =  1 \!\mid\! X = x)$ the conditional
probability of $Y = 1$ given $X = x$. The generalization error for a
surrogate loss $\lbi$ can be rewritten as $ \sR_{\lbi}(h)  = 
\mathbb{E}_{X} \bracket*{\sC_{\lbi}(h, x, x')} $, where
$\sC_{\lbi}(h, x, x')$ is the conditional $\lbi$-risk, defined by
\begin{align*}
\sC_{\lbi}(h, x, x')  = \eta(x)\paren*{1 - \eta(x')} \lbi(h, x, x', + 1,-1)
 + \eta(x')\paren*{1 - \eta(x)}\lbi(h, x, x'-1, + 1).
\end{align*}
We denote by $\sC_{\lbi}^*(\sH,x, x')  =   \inf_{h \in
  \sH}\sC_{\lbi}(h, x, x')$ the minimal conditional $\lbi$-risk. Then,
the minimizability gap can be rewritten as follows:
\begin{align*}
\sM_{\lbi}(\sH)
  =  \sR^*_{\lbi}(\sH) - \mathbb{E}_{X} \bracket* {\sC_{\lbi}^*(\sH,x)}.
\end{align*}
We further refer to $\sC_{\lbi}(h, x, x') - \sC_{\lbi}^*(\sH,x, x')$ as the
calibration gap and denote it by $\Delta\sC_{\lbi,\sH}(h, x, x')$.  For
any $\e  >  0$, we will denote by $\tri*{t}_{\e}$ the $\e$-truncation of
$t  \in \Rset$ defined by $t\mathds{1}_{t > \e}$.

We first prove two general results, which provide bounds between any
loss functions $\sfL_1$ and $\sfL_2$ in both general pairwise ranking
scenario and bipartite ranking scenario.

\begin{restatable}{theorem}{ConBoundPsi}
\label{Thm:bound_Psi}
Assume that there exists a convex function $\Psi\colon
\mathbb{R_{ + }}\to \Rset$ with $\Psi(0)\geq0$ and $\e\geq0$ such that
the following holds for all $h \in \sH$, $x \in \sX$, $x' \in \sX$ and
$\sD \in \sP$:
\begin{equation}
\label{eq:cond_Psi}
\Psi\paren*{\tri*{\Delta\sC_{\sfL_2,\sH}(h, x, x')}_{\e}}\leq \tri*{\Delta\sC_{\sfL_1,\sH}(h, x, x')}_{\e}.
\end{equation}
Then, the following inequality holds for any $h  \in \sH$ and $\sD \in \sP$:
\begin{equation}
\label{eq:bound_Psi_general}
     \Psi\paren*{\sR_{\sfL_2}(h) -  \sR_{\sfL_2}^*(\sH) + \sM_{\sfL_2}(\sH)}
     \leq  \sR_{\sfL_1}(h) - \sR_{\sfL_1}^*(\sH)  + \sM_{\sfL_1}(\sH)  + \max\curl*{\Psi(0),\Psi(\e)}.
\end{equation}
\end{restatable}
\begin{proof}
By the definition of the generalization error and the minimizability
gap, for any $h \in \sH$ and $\sD \in \sP$, we can write the left hand side of \eqref{eq:bound_Psi_general}
as
\begin{align*}
 \Psi\paren*{\sR_{\sfL_2}(h) -  \sR_{\sfL_2}^*(\sH) + \sM_{\sfL_2}(\sH)} 
  & = 
  \Psi\paren*{\sR_{\sfL_2}(h) -  \mathbb{E}_{(X, X')} \bracket* {\sC^*_{\sfL_2}(\sH,x, x')}}\\
  & =  \Psi\paren*{\mathbb{E}_{(X, X')}\bracket*{\Delta\sC_{\sfL_2,\sH}(h, x, x')}}.
\end{align*}
Since $\Psi$ is convex, by Jensen's inequality, it can be upper bounded by $\mathbb{E}_{(X, X')}\bracket*{\Psi\paren*{\Delta\sC_{\sfL_2,\sH}(h, x, x')}}$. Due to the decomposition
\[
\Delta\sC_{\sfL_2,\sH}(h, x, x') = 
\tri*{\Delta\sC_{\sfL_2,\sH}(h, x, x')}_{\e} 
+ \Delta\sC_{\sfL_2,\sH}(h, x, x')\mathds{1}_{\Delta\sC_{\sfL_2,\sH}(h, x, x')\leq \e},
\]
and the assumption $\Psi(0)\geq 0$, we have the following inequality:
\begin{align*}
 \mathbb{E}_{(X, X')}\bracket*{\Psi\paren*{\Delta\sC_{\sfL_2,\sH}(h, x, x')}} 
 & \leq \mathbb{E}_{(X, X')}\bracket*{\Psi\paren*{\tri*{\Delta\sC_{\sfL_2,\sH}(h, x, x')}_{\e}}}\\
 & \qquad +  \mathbb{E}_{(X, X')}\bracket*{\Psi\paren*{\Delta\sC_{\sfL_2,\sH}(h, x, x')\mathds{1}_{\Delta\sC_{\sfL_2,\sH}(h, x, x')\leq \e}}}.
\end{align*}
By assumption \eqref{eq:cond_Psi}, the first term can be bounded as follows:
\begin{align*}
     \mathbb{E}_{(X, X')}\bracket*{\Psi\paren*{\tri*{\Delta\sC_{\sfL_2,\sH}(h, x, x')}_{\e}}} 
     \leq \mathbb{E}_{(X, X')}\bracket*{\Delta\sC_{\sfL_1,\sH}(h, x, x')} 
     = \sR_{\sfL_1}(h) - \sR_{\sfL_1}^*(\sH)  + \sM_{\sfL_1}(\sH).
\end{align*}
Since $\Delta\sC_{\sfL_2,\sH}(h, x, x')\mathds{1}_{\Delta\sC_{\sfL_2,\sH}(h, x, x')\leq \e} \in [0,\e]$, we can bound \[\mathbb{E}_{(X, X')}\bigg[\Psi\bigg(\Delta\sC_{\sfL_2,\sH}(h, x, x')\mathds{1}_{\Delta\sC_{\sfL_2,\sH}(h, x, x')\leq \e}\bigg)\bigg]\] by $\sup_{t \in[0,\e]}\Psi(t)$, which equals $\max\curl*{\Psi(0),\Psi(\e)}$ due to the convexity of $\Psi$.
\end{proof}

\begin{restatable}{theorem}{ConBoundGamma}
\label{Thm:bound_Gamma}
Assume that there exists a non-decreasing concave function $\Gamma\colon
\mathbb{R_{ + }}\to \Rset$ and $\e\geq0$ such that the following holds
for all $h \in \sH$, $x \in \sX$, $x' \in \sX$ and $\sD \in \sP$:
\begin{equation}
\label{eq:cond_Gamma}
\tri*{\Delta\sC_{\sfL_2,\sH}(h, x, x')}_{\e}
\leq \Gamma \paren*{\tri*{\Delta\sC_{\sfL_1,\sH}(h, x, x')}_{\e}}.
\end{equation}
Then, the following inequality holds for any $h  \in \sH$ and $\sD \in \sP$:
    \begin{equation}
    \label{eq:bound_Gamma_general}
    \sR_{\sfL_2}(h) -  \sR_{\sfL_2}^*(\sH)
    \leq  \Gamma\paren*{\sR_{\sfL_1}(h) - \sR_{\sfL_1}^*(\sH) + \sM_{\sfL_1}(\sH)} - \sM_{\sfL_2}(\sH) + \e.
    \end{equation}
\end{restatable}
\begin{proof}
By the definition of the generalization error and the minimizability
gap, for any $h \in \sH$ and $\sD \in \sP$, we can write the left hand side of \eqref{eq:bound_Gamma_general}
as
\begin{align*}
&\sR_{\sfL_2}(h) -  \sR_{\sfL_2}^*(\sH) \\
&  =  \mathbb{E}_{(X, X')}\bracket*{\Delta\sC_{\sfL_2,\sH}(h, x, x')} - \sM_{\sfL_2}(\sH) \\
& =   \mathbb{E}_{(X, X')}\bracket*{\tri*{\Delta\sC_{\sfL_2,\sH}(h, x, x')}_{\e}} + \mathbb{E}_{(X, X')}\bracket*{\Delta\sC_{\sfL_2,\sH}(h, x, x')\mathds{1}_{\Delta\sC_{\sfL_2,\sH}(h, x, x')\leq\e}} - \sM_{\sfL_2}(\sH)\\
\end{align*}
By assumption \eqref{eq:cond_Gamma} and that $\Gamma$ is non-decreasing, the following inequality holds:
\[
\mathbb{E}_{(X, X')}\bracket*{\tri*{\Delta\sC_{\sfL_2,\sH}(h, x, x')}_{\e}} \leq \mathbb{E}_{(X, X')}\bracket*{\Gamma \paren*{\Delta\sC_{\sfL_1,\sH}(h, x, x')}}.
\]
Since $\Gamma$ is concave, by Jensen’s  inequality, 
\begin{align*}
\mathbb{E}_{(X, X')}\bracket*{\Gamma \paren*{\Delta\sC_{\sfL_1,\sH}(h, x, x')}} 
&\leq 
\Gamma\paren*{\mathbb{E}_{(X, X')}\bracket*{\Delta\sC_{\sfL_1,\sH}(h, x, x')}}\\
&= \Gamma\paren*{\sR_{\sfL_1}(h) - \sR_{\sfL_1}^*(\sH) + \sM_{\sfL_1}(\sH)}.
\end{align*}
We complete the proof by noting that $\mathbb{E}_{(X, X')}\bracket*{\Delta\sC_{\sfL_2,\sH}(h, x, x')\mathds{1}_{\Delta\sC_{\sfL_2,\sH}(h, x, x')\leq\e}} \leq \e$.
\end{proof}

\section{\texorpdfstring{$\sH$}{H}-consistency bounds for general pairwise ranking with abstention (Proof of Theorem~\ref{thm:general-bound})}
\label{app:abstention-general}

We first characterize the minimal conditional $\labs$-risk and the
calibration gap of $\labs$ for a broad class of hypothesis sets. We
let $\ov \sH(x, x')  =  \curl[big]{h  \in \sH \colon \sign(h(x') - h(x))
  \paren*{2\eta(x, x') - 1} \leq 0}$ for convenience.

\begin{restatable}{lemma}{CalibrationGapAbsGeneral}
\label{lemma:calibration-gap-abs-general}
Assume that $\sH$ is regular for general  pairwise  ranking.
Then, the minimal conditional $\labs$-risk is
\begin{align*}
\sC^*_{\labs}(\sH,x, x') = \min\curl*{\eta(x, x'), 1 - \eta(x, x')}\1_{\norm*{x - x'}  >  \gamma} +  c\, \1_{\abs*{x - x'} \leq \gamma}.
\end{align*}
The calibration gap of $\labs$ can be characterized as
\begin{align*}
 \Delta\sC_{\labs,\sH}(h, x, x') = \abs*{2\eta(x, x') - 1}\mathds{1}_{h  \in \ov \sH(x, x')}\1_{\norm*{x - x'}  >  \gamma}.
\end{align*}
\end{restatable}
\begin{proof}
By the definition, the conditional $\labs$-risk is 
\begin{align*}
\sC_{\labs}(h, x, x') 
&  =   \paren*{\eta(x, x')\mathds{1}_{h(x') <  h(x)} + (1 -  \eta(x, x'))\mathds{1}_{h(x') \geq h(x)}}\1_{\norm*{x - x'}  >  \gamma} +  c\, \1_{\abs*{x - x'} \leq \gamma}.
\end{align*}
For any $(x, x')$ such that $\norm*{x - x'}\leq\gamma$ and $h \in \sH$,
$\sC_{\labs}(h, x, x) = \sC^*_{\labs}(\sH,x, x) = c$. For any $(x, x')$ such
that $\norm*{x - x'} > \gamma$, by the assumption, there exists $h^* \in
\sH$ such that
$\sign(h^*(x') - h^*(x)) = \sign\paren*{2\eta(x, x') - 1}$. Therefore, the
optimal conditional $\labs$-risk can be characterized as for any
$x, x' \in \sX$,
\begin{align*}
\sC^*_{\labs}(\sH,x, x') =  \sC_{\labs}\paren*{h^*,x, x'} = \min\curl*{\eta(x, x'), 1 - \eta(x, x')}\1_{\norm*{x - x'}  >  \gamma} +  c\, \1_{\abs*{x - x'} \leq \gamma}.
\end{align*}
which proves the first part of lemma. By the definition, for any
$(x, x')$ such that $\norm*{x - x'}\leq\gamma$ and $h \in \sH$,
$\Delta\sC_{\labs,\sH}(h, x, x') = \sC_{\labs}(h, x, x') -
\sC^*_{\labs}(\sH,x, x') = 0$. . For any $(x, x')$ such that
$\norm*{x - x'} > \gamma$ and $h \in \sH$,
\begin{align*}
&\Delta\sC_{\labs,\sH}(h, x, x')\\
& = \sC_{\labs}(h, x, x') - \sC^*_{\labs}(\sH,x, x')\\
&  =  \eta(x, x')\mathds{1}_{h(x') <  h(x)} + (1 -  \eta(x, x'))\mathds{1}_{h(x') \geq h(x)} - \min\curl*{\eta(x, x'), 1 - \eta(x, x')}\\
&  =  
\begin{cases}
\abs*{2\eta(x, x') - 1},  & h  \in \ov\sH(x, x'), \\
0, & \text{otherwise} .
\end{cases}
\end{align*}
This leads to
\begin{align*}
\Delta\sC_{\labs,\sH}(h, x, x') = \abs*{2\eta(x, x') - 1}\mathds{1}_{h  \in \ov \sH(x, x')}\1_{\norm*{x - x'}  >  \gamma}.
\end{align*}
\end{proof}

\GeneralBound*
\begin{proof}
Since $\sH_{\mathrm{lin}}$ and $\sH_{\mathrm{NN}}$ satisfy the condition of
Lemma~\ref{lemma:calibration-gap-abs-general}, by
Lemma~\ref{lemma:calibration-gap-abs-general} the
$\paren*{\labs,\sH_{\mathrm{lin}}}$-minimizability gap and the $\paren*{\labs,\sH_{\mathrm{NN}}}$-minimizability gap  can be
expressed as follows:
\begin{align*}
\sM_{\labs}(\sH_{\mathrm{lin}})  &=  \sR_{\labs}^*(\sH_{\mathrm{lin}}) - \mathbb{E}_{(X, X')}\bracket*{\min\curl*{\eta(x, x'), 1 - \eta(x, x')}\1_{\norm*{x - x'}  >  \gamma} +  c\, \1_{\abs*{x - x'} \leq \gamma}}\\
\sM_{\labs}(\sH_{\mathrm{NN}})  &=  \sR_{\labs}^*(\sH_{\mathrm{NN}}) - \mathbb{E}_{(X, X')}\bracket*{\min\curl*{\eta(x, x'), 1 - \eta(x, x')}\1_{\norm*{x - x'}  >  \gamma} +  c\, \1_{\abs*{x - x'} \leq \gamma}.}.
\end{align*}
By the definition of $\sH_{\mathrm{lin}}$ and $\sH_{\mathrm{NN}}$, for any $(x, x')  \in \sX\times\sX$, $\curl[\big]{h(x') - h(x) \mid h  \in \sH_{\mathrm{lin}}}  =  \bracket*{-W \norm*{x - x'}_p, W\norm*{x - x'}_p}$ and $\curl[\big]{h(x') - h(x) \mid h  \in \sH_{\mathrm{NN}}}  = 
\bracket*{-\Lambda W \norm*{x - x'}_p, \Lambda W\norm*{x - x'}_p}$. In the following, we will prove the bounds for $\sH_{\mathrm{lin}}$. Similar proofs with $B$ replaced by $\Lambda B$ hold for $\sH_{\mathrm{NN}}$.

\paragraph{Proof for \texorpdfstring{$\sfL_{\Phi_{\mathrm{hinge}}}$}{hinge}.} For the hinge loss function $\Phi_{\mathrm{hinge}}(u)\colon =
\max\curl*{0, 1 - u}$, for all $h \in \sH_{\mathrm{lin}}$ and $(x,
x')$ such that $\norm*{x - x'}_p > \gamma$,
\begin{equation*}
\begin{aligned}
& \sC_{\sfL_{\Phi_{\mathrm{hinge}}}}(h, x, x')\\
&  = \eta(x, x') \sfL_{\Phi_{\mathrm{hinge}}}(h(x') - h(x)) + (1 - \eta(x, x'))\sfL_{\Phi_{\mathrm{hinge}}}(h(x) - h(x'))\\
&  = \eta(x, x')\max\curl*{0, 1 - h(x') + h(x)} + (1 - \eta(x, x'))\max\curl*{0, 1 + h(x') - h(x)}.
\end{aligned}
\end{equation*}
Then,
\begin{align*}
\sC^*_{\sfL_{\Phi_{\mathrm{hinge}}},\sH_{\mathrm{lin}}}(x, x') =  \inf_{h \in\sH_{\mathrm{lin}}}\sC_{\sfL_{\Phi_{\mathrm{hinge}}}}(h, x, x')
  =  1 - \abs*{2\eta(x, x') - 1}\min\curl*{W\norm*{x - x'}_p, 1}.
\end{align*}
The $\paren*{\sfL_{\Phi_{\mathrm{hinge}}},\sH_{\mathrm{lin}}}$-minimizability gap is
\begin{equation}
\begin{aligned}
\label{eq:M-hinge-lin}
\sM_{\sfL_{\Phi_{\mathrm{hinge}}}}(\sH_{\mathrm{lin}})
&  =  \sR_{\sfL_{\Phi_{\mathrm{hinge}}}}^*(\sH_{\mathrm{lin}}) - \mathbb{E}_{(X, X')}\bracket*{\sC^*_{\sfL_{\Phi_{\mathrm{hinge}}},\sH_{\mathrm{lin}}}(x, x')}\\
&  =  \sR_{\sfL_{\Phi_{\mathrm{hinge}}}}^*(\sH_{\mathrm{lin}}) - \mathbb{E}_{(X, X')}\bracket*{1 - \abs*{2\eta(x, x') - 1}\min\curl*{W\norm*{x - x'}_p, 1}}.
\end{aligned}
\end{equation}
Therefore, $\forall h  \in \ov\sH_{\mathrm{lin}}(x, x')$,
\begin{align*} &\Delta\sC_{\sfL_{\Phi_{\mathrm{hinge}}},\sH_{\mathrm{lin}}}(h, x, x')\\
& \geq  \inf_{h \in\ov\sH_{\mathrm{lin}}(x, x')}\sC_{\sfL_{\Phi_{\mathrm{hinge}}}}(h, x, x') - \sC^*_{\sfL_{\Phi_{\mathrm{hinge}}},\sH_{\mathrm{lin}}}(x, x')\\
&  =  \eta(x, x')\max\curl*{0, 1 - 0} + (1 - \eta(x, x'))\max\curl*{0, 1 + 0} - \sC^*_{\sfL_{\Phi_{\mathrm{hinge}}},\sH_{\mathrm{lin}}}(x, x')\\
&  =  1 -  \bracket*{1 - \abs*{2\eta(x, x') - 1}\min\curl*{W\norm*{x - x'}_p, 1}}\\
& = \abs*{2\eta(x, x') - 1}\min\curl*{W\norm*{x - x'}_p, 1}\\
& \geq \abs*{2\eta(x, x') - 1}\min\curl*{W\gamma, 1}
\end{align*}
which implies that for any $h \in \sH_{\mathrm{lin}}$ and $(x, x')$ such that $\norm*{x - x'}_p > \gamma$,
\begin{align*}
\Delta\sC_{\sfL_{\Phi_{\mathrm{hinge}}},\sH_{\mathrm{lin}}}(h, x, x') \geq \min\curl*{W\gamma, 1}\tri*{\abs*{2\eta(x, x') - 1}}_{0}\mathds{1}_{h  \in \ov \sH_{\mathrm{lin}}(x, x')} = \Delta\sC_{\labs,\sH_{\mathrm{lin}}}(h, x, x').
\end{align*}
Thus, by Theorem~\ref{Thm:bound_Psi} or Theorem~\ref{Thm:bound_Gamma}, setting $\e  =  0$ yields the $\sH_{\mathrm{lin}}$-consistency bound for $\sfL_{\Phi_{\mathrm{hinge}}}$, valid for all $h  \in \sH_{\mathrm{lin}}$:
\begin{align}
\label{eq:bound-hinge-lin}
     \sR_{\labs}(h) -  \sR_{\labs}^*(\sH_{\mathrm{lin}})
     \leq \frac{\sR_{\sfL_{\Phi_{\mathrm{hinge}}}}(h) -  \sR_{\sfL_{\Phi_{\mathrm{hinge}}}}^*(\sH_{\mathrm{lin}}) + \sM_{\sfL_{\Phi_{\mathrm{hinge}}}}(\sH_{\mathrm{lin}})}{\min\curl*{W\gamma, 1}} - \sM_{\labs}(\sH_{\mathrm{lin}}).
\end{align}

\paragraph{Proof for \texorpdfstring{$\sfL_{\Phi_{\mathrm{exp}}}$}{exp}.} For the exponential loss function $\Phi_{\mathrm{exp}}(u)\colon =
e^{-u}$, for all $h \in \sH_{\mathrm{lin}}$ and $(x, x')$ such that
$\norm*{x - x'}_p > \gamma$,
\begin{equation*}
\begin{aligned}
\sC_{\sfL_{\Phi_{\mathrm{exp}}}}(h, x, x')
& = \eta(x, x') \sfL_{\Phi_{\mathrm{exp}}}(h(x') - h(x)) + (1 - \eta(x, x'))\sfL_{\Phi_{\mathrm{exp}}}(h(x) - h(x'))\\
& = \eta(x, x') e^{-h(x') + h(x)} + (1 - \eta(x, x'))e^{h(x') - h(x)}.
\end{aligned}
\end{equation*}
Then, 
\begin{align*}
&\sC^*_{\sfL_{\Phi_{\mathrm{exp}}}, \sH_{\mathrm{lin}}}(x, x')\\
& =   \inf_{h \in\sH_{\mathrm{lin}}}\sC_{\sfL_{\Phi_{\mathrm{exp}}}}(h, x, x')\\
& =  
\begin{cases}
2\sqrt{\eta(x, x')(1 - \eta(x, x'))}  \\ \text{if } \frac{1}{2}\abs*{\log\frac{\eta(x, x')}{1 - \eta(x, x')}}\leq W\norm*{x - x'}_p\\
\max\curl*{\eta(x, x'), 1 - \eta(x, x')}e^{-W\norm*{x - x'}_p} + \min\curl*{\eta(x, x'), 1 - \eta(x, x')}e^{W\norm*{x - x'}_p} \\ \text{if } \frac{1}{2}\abs*{\log\frac{\eta(x, x')}{1 - \eta(x, x')}} >  W\norm*{x - x'}_p.
\end{cases}
\end{align*}
The $\paren*{\sfL_{\Phi_{\mathrm{exp}}},\sH_{\mathrm{lin}}}$-minimizability gap is:
\begin{equation}
\begin{aligned}
\label{eq:M-exp-lin}
\sM_{\sfL_{\Phi_{\mathrm{exp}}}}(\sH_{\mathrm{lin}})
&  =  \sR_{\sfL_{\Phi_{\mathrm{exp}}}}^*(\sH_{\mathrm{lin}}) - 
\mathbb{E}_{(X, X')}\bracket*{\sC^*_{\sfL_{\Phi_{\mathrm{exp}}}, \sH_{\mathrm{lin}}}(x, x')}\\
&  =  \sR_{\sfL_{\Phi_{\mathrm{exp}}}}^*(\sH_{\mathrm{lin}}) - 
\mathbb{E}_{(X, X')}\bracket*{2\sqrt{\eta(x, x')(1 - \eta(x, x'))}\1_{\frac{1}{2}\abs*{\log\frac{\eta(x, x')}{1 - \eta(x, x')}}\leq W\norm*{x - x'}_p}}\\
& \qquad - \mathbb{E}_{(X, X')}\bracket*{\max\curl*{\eta(x, x'), 1 - \eta(x, x')}e^{-W\norm*{x - x'}_p}\1_{\frac{1}{2}\abs*{\log\frac{\eta(x, x')}{1 - \eta(x, x')}} >  W\norm*{x - x'}_p}}\\
& \quad \qquad - \mathbb{E}_{(X, X')}\bracket*{\min\curl*{\eta(x, x'), 1 - \eta(x, x')}e^{W\norm*{x - x'}_p}\1_{\frac{1}{2}\abs*{\log\frac{\eta(x, x')}{1 - \eta(x, x')}} >  W\norm*{x - x'}_p}}.
\end{aligned}
\end{equation}
Therefore, $\forall h  \in \ov\sH_{\mathrm{lin}}(x, x')$,
\begin{align*} 
&\Delta\sC_{\sfL_{\Phi_{\mathrm{exp}}},\sH_{\mathrm{lin}}}(h, x, x')\\
& \geq  \inf_{h \in\ov\sH_{\mathrm{lin}}(x, x')}\sC_{\sfL_{\Phi_{\mathrm{exp}}}}(h, x, x') - \sC^*_{\sfL_{\Phi_{\mathrm{exp}}}, \sH_{\mathrm{lin}}}(x, x')\\
&  =  \eta(x, x')e^{-0} + (1 - \eta(x, x'))e^{0} - \sC^*_{\sfL_{\Phi_{\mathrm{exp}}}, \sH_{\mathrm{lin}}}(x, x')\\
&  =  
\begin{cases}
1 -  2\sqrt{\eta(x, x')(1 - \eta(x, x'))} \\ \text{if }\frac{1}{2}\abs*{\log\frac{\eta(x, x')}{1 - \eta(x, x')}}\leq W\norm*{x - x'}_p\\
1 -  \max\curl*{\eta(x, x'), 1 - \eta(x, x')}e^{-W\norm*{x - x'}_p} - \min\curl*{\eta(x, x'), 1 - \eta(x, x')}e^{W\norm*{x - x'}_p} \\ \text{if }\frac{1}{2}\abs*{\log\frac{\eta(x, x')}{1 - \eta(x, x')}} >  W\norm*{x - x'}_p
\end{cases}\\
& \geq 
\begin{cases}
1 -  2\sqrt{\eta(x, x')(1 - \eta(x, x'))} \\ \text{if } \frac{1}{2}\abs*{\log\frac{\eta(x, x')}{1 - \eta(x, x')}}\leq W\gamma\\
1 -  \max\curl*{\eta(x, x'), 1 - \eta(x, x')}e^{-W\gamma} - \min\curl*{\eta(x, x'), 1 - \eta(x, x')}e^{W\gamma} \\ \text{if } \frac{1}{2}\abs*{\log\frac{\eta(x, x')}{1 - \eta(x, x')}} >  W\gamma
\end{cases}\\
&  =  \Psi_{\rm{exp}} \paren*{\abs*{2\eta(x, x') - 1}},
\end{align*}
where $\Psi_{\rm{exp}}$ is the increasing and convex function on $[0, 1]$ defined by
\begin{align*}
\forall t \in[0, 1], \quad 
\Psi_{\rm{exp}}(t) = \begin{cases}
1 - \sqrt{1 - t^2}, & t\leq \frac{e^{2W\gamma} - 1}{e^{2W\gamma} + 1}\\
1 - \frac{t + 1}{2}e^{-W\gamma} - \frac{1 - t}{2}e^{W\gamma}, & t >  \frac{e^{2W\gamma} - 1}{e^{2W\gamma} + 1}
\end{cases}
\end{align*}
which implies that for any $h \in \sH_{\mathrm{lin}}$ and $(x, x')$ such that $\norm*{x - x'}_p > \gamma$,
\begin{align*}
\Delta\sC_{\sfL_{\Phi_{\mathrm{exp}}},\sH_{\mathrm{lin}}}(h, x, x') \geq \Psi_{\rm{exp}}\paren*{\Delta\sC_{\labs,\sH_{\mathrm{lin}}}(h, x, x')}.
\end{align*}
To simplify the expression, using the fact that
\begin{align*}
1 -  \sqrt{1 - t^2} & \geq \frac{t^2}{2}, \\
1 - \frac{t + 1}{2}e^{-W\gamma} - \frac{1 - t}{2}e^{W\gamma} &  =  1 - \frac{e^{W\gamma}}2 - \frac{e^{-W\gamma}}2 + \frac{e^{W\gamma} - e^{-W\gamma}}2\, t,
\end{align*}
$\Psi_{\rm{exp}}$ can be lower bounded by
\begin{align*}
\wt\Psi_{\rm{exp}}(t) =  \begin{cases}
\frac{t^2}{2},& t\leq \frac{e^{2W\gamma} - 1}{e^{2W\gamma} + 1}\\
\frac{1}{2}\paren*{\frac{e^{2W\gamma} - 1}{e^{2W\gamma} + 1}}\, t, & t >  \frac{e^{2W\gamma} - 1}{e^{2W\gamma} + 1}.
\end{cases}   
\end{align*}
Thus, we adopt an upper bound of $\Psi^{-1}$ as follows:
\begin{align*}
\Gamma_{\Phi_{\mathrm{exp}}}(t) = \wt\Psi_{\rm{exp}}^{-1}(t)& = 
\begin{cases}
\sqrt{2t}, & t\leq \frac{1}{2}\paren*{\frac{e^{2W\gamma} - 1}{e^{2W\gamma} + 1}}^2\\
2\paren*{\frac{e^{2W\gamma} + 1}{e^{2W\gamma} - 1}}\, t, & t >  \frac{1}{2}\paren*{\frac{e^{2W\gamma} - 1}{e^{2W\gamma} + 1}}^2
\end{cases}\\
& = \max\curl*{\sqrt{2t},2\paren*{\frac{e^{2W\gamma} + 1}{e^{2W\gamma} - 1}}\, t}.
\end{align*}
Thus, by Theorem~\ref{Thm:bound_Psi} or Theorem~\ref{Thm:bound_Gamma}, setting $\e  =  0$ yields the $\sH_{\mathrm{lin}}$-consistency bound for $\sfL_{\Phi_{\mathrm{exp}}}$, valid for all $h  \in \sH_{\mathrm{lin}}$:
\begin{align}
\label{eq:bound-exp-lin}
     \sR_{\labs}(h) -  \sR_{\labs}^*(\sH_{\mathrm{lin}})
     \leq  \Gamma_{\Phi_{\mathrm{exp}}}\paren*{\sR_{\sfL_{\Phi_{\mathrm{exp}}}}(h) -  \sR_{\sfL_{\Phi_{\mathrm{exp}}}}^*(\sH_{\mathrm{lin}}) + \sM_{\sfL_{\Phi_{\mathrm{exp}}}}(\sH_{\mathrm{lin}})} - \sM_{\labs}(\sH_{\mathrm{lin}}).
\end{align}
where $\Gamma_{\Phi_{\mathrm{exp}}}(t) = \max\curl*{\sqrt{2t},2\paren*{\frac{e^{2W\gamma} + 1}{e^{2W\gamma} - 1}}\, t}$.

\paragraph{Proof for \texorpdfstring{$\sfL_{\Phi_{\mathrm{sig}}}$}{sig}.} For the sigmoid loss function $\Phi_{\mathrm{sig}}(u)\colon = 1 -
\tanh(k u),~k > 0$, for all $h \in \sH_{\mathrm{lin}}$ and $(x, x')$
such that $\norm*{x - x'}_p > \gamma$,
\begin{equation*}
\begin{aligned}
& \sC_{\sfL_{\Phi_{\mathrm{sig}}}}(h, x, x')\\
& = \eta(x, x') \sfL_{\Phi_{\mathrm{sig}}}(h(x') - h(x)) + (1 - \eta(x, x'))\sfL_{\Phi_{\mathrm{sig}}}(h(x) - h(x'))\\
& = \eta(x, x') \paren*{1 - \tanh\paren*{k\bracket*{h(x') - h(x)}}} + (1 - \eta(x, x'))\paren*{1 + \tanh\paren*{k\bracket*{h(x') - h(x)}}}.
\end{aligned}
\end{equation*}
Then, 
\begin{align*}
\sC^*_{\sfL_{\Phi_{\mathrm{sig}}}}(\sH_{\mathrm{lin}})(x, x')
&  =   \inf_{h \in\sH_{\mathrm{lin}}}\sC_{\sfL_{\Phi_{\mathrm{sig}}}}(h, x, x') = 1 - \abs*{1 - 2\eta(x, x')}\tanh\paren*{kW\norm*{x - x'}_p}.
\end{align*}
The $\paren*{\sfL_{\Phi_{\mathrm{sig}}},\sH_{\mathrm{lin}}}$-minimizability gap is:
\begin{equation}
\begin{aligned}
\label{eq:M-sig-lin}
\sM_{\sfL_{\Phi_{\mathrm{sig}}}}(\sH_{\mathrm{lin}})
&  =  \sR_{\sfL_{\Phi_{\mathrm{sig}}}}^*(\sH_{\mathrm{lin}}) - 
\mathbb{E}_{(X, X')}\bracket*{\sC^*_{\sfL_{\Phi_{\mathrm{sig}}}, \sH_{\mathrm{lin}}}(x, x')}\\
&  =  \sR_{\sfL_{\Phi_{\mathrm{sig}}}}^*(\sH_{\mathrm{lin}}) - 
\mathbb{E}_{(X, X')}\bracket*{1 - \abs*{1 - 2\eta(x, x')}\tanh\paren*{kW\norm*{x - x'}_p}}.
\end{aligned}
\end{equation}
Therefore, $\forall h  \in \ov\sH_{\mathrm{lin}}(x, x')$,
\begin{align*} \Delta\sC_{\sfL_{\Phi_{\mathrm{sig}}},\sH_{\mathrm{lin}}}(h, x, x')
& \geq  \inf_{h \in\ov\sH_{\mathrm{lin}}(x, x')}\sC_{\sfL_{\Phi_{\mathrm{sig}}}}(h, x, x') - \sC^*_{\sfL_{\Phi_{\mathrm{sig}}}, \sH_{\mathrm{lin}}}(x, x')\\
&  = 1 - \abs*{1 - 2\eta(x, x')}\tanh(0) -  \sC^*_{\sfL_{\Phi_{\mathrm{sig}}}, \sH_{\mathrm{lin}}}(x, x')\\
&  =  \abs*{1 - 2\eta(x, x')}\tanh\paren*{kW\norm*{x - x'}_p}\\
& \geq \abs*{1 - 2\eta(x, x')}\tanh\paren*{kW\gamma}
\end{align*}
which implies that for any $h \in \sH_{\mathrm{lin}}$ and $(x, x')$ such that $\norm*{x - x'}_p > \gamma$,
\begin{align*}
\Delta\sC_{\sfL_{\Phi_{\mathrm{sig}}},\sH_{\mathrm{lin}}}(h, x, x') \geq \tanh\paren*{kW\gamma}\Delta\sC_{\labs,\sH_{\mathrm{lin}}}(h, x, x').
\end{align*}
Thus, by Theorem~\ref{Thm:bound_Psi} or Theorem~\ref{Thm:bound_Gamma}, setting $\e  =  0$ yields the $\sH_{\mathrm{lin}}$-consistency bound for $\sfL_{\Phi_{\mathrm{sig}}}$, valid for all $h  \in \sH_{\mathrm{lin}}$:
\begin{align}
\label{eq:bound-sig-lin}
     \sR_{\labs}(h) -  \sR_{\labs}^*(\sH_{\mathrm{lin}})
     \leq \frac{\sR_{\sfL_{\Phi_{\mathrm{sig}}}}(h) -  \sR_{\sfL_{\Phi_{\mathrm{sig}}}}^*(\sH_{\mathrm{lin}}) + \sM_{\sfL_{\Phi_{\mathrm{sig}}}}(\sH_{\mathrm{lin}})}{\tanh\paren*{kW\gamma}} - \sM_{\labs}(\sH_{\mathrm{lin}}).
\end{align}
\end{proof}

\section{\texorpdfstring{$\sH$}{H}-consistency bounds for bipartite ranking with abstention (Proof of Theorem~\ref{thm:bi-bound})}
\label{app:abstention-bi}
We first characterize the minimal conditional $\labsbi$-risk and the calibration gap of $\labsbi$ for a broad class of hypothesis sets. We
let $\wt \sH(x, x')  = 
\curl[big]{h  \in \sH \colon (h(x) - h(x'))(\eta(x) - \eta(x')) <  0}$ and $\mathring{\sH}(x, x')  = 
\curl[big]{h  \in \sH\colon h(x) = h(x')}$ for convenience.

\begin{restatable}{lemma}{CalibrationGapAbsBi}
\label{lemma:calibration-gap-abs-bi}
Assume that $\sH$ is regular for bipartite ranking.
Then, the minimal conditional $\labsbi$-risk is
\begin{align*}
\sC^*_{\labsbi}(\sH,x, x') = \min\curl*{\eta(x)(1 - \eta(x')),\eta(x')(1 - \eta(x))}\1_{\norm*{x - x'}  >  \gamma} +  c\, \1_{\abs*{x - x'} \leq \gamma}.
\end{align*}
The calibration gap of $\labsbi$ can be characterized as
\begin{align*}
\Delta\sC_{\labsbi,\sH}(h, x, x')   = \abs*{\eta(x) - \eta(x')}\mathds{1}_{h  \in \wt \sH(x, x')}\1_{\norm*{x - x'}  >  \gamma}  + \frac{1}{2}\abs*{\eta(x) - \eta(x')}\mathds{1}_{h  \in \mathring{\sH}(x, x')}\1_{\norm*{x - x'}  >  \gamma}.
\end{align*}
\end{restatable}
\begin{proof}
By the definition, the conditional $\labsbi$-risk is 
\begin{align*}
\sC_{\labsbi}(h, x, x') 
&=   \bigg(\eta(x)(1 - \eta(x')) \bracket*{\mathds{1}_{h(x) - h(x') < 0} + \frac{1}{2}\mathds{1}_{h(x) = h(x')}}\\
& \qquad +\eta(x')(1 - \eta(x)) \bracket*{\mathds{1}_{h(x) - h(x') > 0} + \frac{1}{2}\mathds{1}_{h(x) = h(x')}}\bigg)\1_{\norm*{x - x'}  >  \gamma}
 +  c\, \1_{\abs*{x - x'} \leq \gamma}.
\end{align*}
For any $(x, x')$ such that $\norm*{x - x'}\leq\gamma$ and $h \in \sH$, $\sC_{\labsbi}(h, x, x) = \sC^*_{\labsbi}(\sH,x, x) = c$. For any $(x, x')$ such that $\norm*{x - x'} > \gamma$, by the assumption, there exists $h^* \in \sH$ such that
\begin{align*}
\paren*{h^*(x) - h^*(x')}\paren*{\eta(x) - \eta(x')}\mathds{1}_{\eta(x) \neq \eta(x')} > 0.
\end{align*}
Therefore, the optimal conditional $\labsbi$-risk can be characterized as for any $x, x' \in \sX$,
\begin{align*}
\sC^*_{\labsbi}(\sH,x, x') 
&=  \sC_{\labsbi}\paren*{h^*,x, x'}\\
&= \min\curl*{\eta(x)(1 - \eta(x')),\eta(x')(1 - \eta(x))}\1_{\norm*{x - x'}  >  \gamma} +  c\, \1_{\abs*{x - x'} \leq \gamma}.
\end{align*}
which proves the first part of lemma. By the definition, for any $(x, x')$ such that $\norm*{x - x'}\leq\gamma$ and $h \in \sH$, $\Delta\sC_{\labsbi,\sH}(h, x, x') = \sC_{\labsbi}(h, x, x') - \sC^*_{\labsbi}(\sH,x, x') = 0$. . For any $(x, x')$ such that $\norm*{x - x'} > \gamma$ and $h \in \sH$,
\begin{align*}
\Delta\sC_{\labsbi,\sH}(h, x, x')
& = \sC_{\labsbi}(h, x, x') - \sC^*_{\labsbi}(\sH,x, x')\\
&  =  \eta(x)(1 - \eta(x')) \bracket*{\mathds{1}_{h(x) - h(x') < 0}
  + \frac{1}{2}\mathds{1}_{h(x) = h(x')}}\\
&  + \eta(x')(1 - \eta(x)) \bracket*{\mathds{1}_{h(x) - h(x') > 0}
  + \frac{1}{2}\mathds{1}_{h(x) = h(x')}}\\
& - \min\curl*{\eta(x)(1 - \eta(x')),\eta(x')(1 - \eta(x))}\\
&  =  
\begin{cases}
\abs*{\eta(x)(1 - \eta(x')) - \eta(x')(1 - \eta(x))},  & h  \in \wt\sH(x, x'), \\
\frac{1}{2}\abs*{\eta(x)(1 - \eta(x')) - \eta(x')(1 - \eta(x))},
& h  \in \mathring{\sH}(x, x'), \\
0, & \text{otherwise} .
\end{cases}\\
&  =  
\begin{cases}
\abs*{\eta(x) - \eta(x')},  & h  \in \wt\sH(x, x'), \\
\frac{1}{2}\abs*{\eta(x) - \eta(x')},  & h  \in \mathring{\sH}(x, x'), \\
0, & \text{otherwise} .
\end{cases}
\end{align*}
This leads to
\begin{align*}
\tri*{\Delta\sC_{\labsbi,\sH}(h, x, x')}_{\e}
&  = \tri*{\abs*{\eta(x) - \eta(x')}}_{\e}\mathds{1}_{h  \in \wt \sH(x, x')}\1_{\norm*{x - x'}  >  \gamma}\\
&\qquad + \tri*{\frac{1}{2}\abs*{\eta(x) - \eta(x')}}_{\e}\mathds{1}_{h  \in \mathring{\sH}(x, x')}\1_{\norm*{x - x'}  >  \gamma}.
\end{align*}
\end{proof}

\BiBound*
\begin{proof}
Since $\sH_{\mathrm{lin}}$ and $\sH_{\mathrm{NN}}$ satisfy the condition of Lemma~\ref{lemma:calibration-gap-abs-bi}, by Lemma~\ref{lemma:calibration-gap-abs-bi} the $\paren*{\labsbi,\sH_{\mathrm{lin}}}$-minimizability gap and the $\paren*{\labsbi,\sH_{\mathrm{NN}}}$-minimizability gap can be expressed as follows:
\begin{align*}
\sM_{\labsbi}(\sH_{\mathrm{lin}})  &=  \sR_{\labsbi}^*(\sH_{\mathrm{lin}})\\
&\qquad - \mathbb{E}_{(X, X')}\bracket*{\min\curl*{\eta(x)(1 - \eta(x')),\eta(x')(1 - \eta(x))}\1_{\norm*{x - x'}  >  \gamma} +  c\, \1_{\abs*{x - x'} \leq \gamma}}\\
\sM_{\labsbi}(\sH_{\mathrm{NN}})  &=  \sR_{\labsbi}^*(\sH_{\mathrm{NN}})\\
&\qquad - \mathbb{E}_{(X, X')}\bracket*{\min\curl*{\eta(x)(1 - \eta(x')),\eta(x')(1 - \eta(x))}\1_{\norm*{x - x'}  >  \gamma} +  c\, \1_{\abs*{x - x'} \leq \gamma}}.
\end{align*}
By the definition of $\sH_{\mathrm{lin}}$ and $\sH_{\mathrm{NN}}$, for any $(x, x')  \in \sX\times\sX$, $\curl[\big]{h(x') - h(x) \mid h  \in \sH_{\mathrm{lin}}}  =  \bracket*{-W \norm*{x - x'}_p, W\norm*{x - x'}_p}$ and $\curl[\big]{h(x') - h(x) \mid h  \in \sH_{\mathrm{NN}}}  =  \bracket*{-\Lambda W \norm*{x - x'}_p, \Lambda W\norm*{x - x'}_p}$. In the following, we will prove the bounds for $\sH_{\mathrm{lin}}$. Similar proofs with $B$ replaced by $\Lambda B$ hold for $\sH_{\mathrm{NN}}$.

\paragraph{Proof for \texorpdfstring{$\lbi_{\Phi_{\mathrm{hinge}}}$}{hinge}.} For the hinge loss function $\Phi_{\mathrm{hinge}}(u)\colon =
\max\curl*{0, 1 - u}$, for all $h \in \sH_{\mathrm{lin}}$ and $(x,
x')$ such that $\norm*{x - x'}_p > \gamma$,
\begin{equation*}
\begin{aligned}
& \sC_{\lbi_{\Phi_{\mathrm{hinge}}}}(h, x, x')\\
&  = \eta(x)(1 - \eta(x')) \Phi_{\mathrm{hinge}}(h(x) - h(x')) + \eta(x')(1 - \eta(x))\Phi_{\mathrm{hinge}}(h(x') - h(x))\\
&  = \eta(x)(1 - \eta(x'))\max\curl*{0, 1 - h(x) + h(x')} + \eta(x')(1 - \eta(x))\max\curl*{0, 1 + h(x) - h(x')}.
\end{aligned}
\end{equation*}
Then,
\begin{align*}
& \sC^*_{\lbi_{\Phi_{\mathrm{hinge}}},\sH_{\mathrm{lin}}}(x, x')\\
& =  \inf_{h \in\sH_{\mathrm{lin}}}\sC_{\lbi_{\Phi_{\mathrm{hinge}}}}(h, x, x')\\
&=  \eta(x)(1 - \eta(x')) + \eta(x')(1 - \eta(x)) - \abs*{\eta(x) - \eta(x')}\min\curl*{W\norm*{x - x'}_p, 1}.
\end{align*}
The $\paren*{\lbi_{\Phi_{\mathrm{hinge}}},\sH_{\mathrm{lin}}}$-minimizability gap is
\begin{equation*}
\begin{aligned}
& \sM_{\lbi_{\Phi_{\mathrm{hinge}}}}(\sH_{\mathrm{lin}})\\
&  =  \sR_{\lbi_{\Phi_{\mathrm{hinge}}}}^*(\sH_{\mathrm{lin}}) - \mathbb{E}_{(X, X')}\bracket*{\sC^*_{\lbi_{\Phi_{\mathrm{hinge}}},\sH_{\mathrm{lin}}}(x, x')}\\
&  =  \sR_{\lbi_{\Phi_{\mathrm{hinge}}}}^*(\sH_{\mathrm{lin}})\\
& \qquad - \mathbb{E}_{(X, X')}\bracket*{\eta(x)(1 - \eta(x')) + \eta(x')(1 - \eta(x)) - \abs*{\eta(x) - \eta(x')}\min\curl*{W\norm*{x - x'}_p, 1}}.
\end{aligned}
\end{equation*}
Therefore, $\forall h  \in \wt\sH_{\mathrm{lin}}(x, x')\bigcup \mathring{\sH}_{\mathrm{lin}}(x, x')$,
\begin{align*} & \Delta\sC_{\lbi_{\Phi_{\mathrm{hinge}}},\sH_{\mathrm{lin}}}(h, x, x')\\
& \geq  \inf_{h \in\wt\sH_{\mathrm{lin}}(x, x')\bigcup \mathring{\sH}_{\mathrm{lin}}(x, x')}\sC_{\lbi_{\Phi_{\mathrm{hinge}}}}(h, x, x') - \sC^*_{\lbi_{\Phi_{\mathrm{hinge}}},\sH_{\mathrm{lin}}}(x, x')\\
&  =  \eta(x)(1 - \eta(x'))\max\curl*{0, 1 - 0} + \eta(x')(1 - \eta(x))\max\curl*{0, 1 + 0} - \sC^*_{\lbi_{\Phi_{\mathrm{hinge}}},\sH_{\mathrm{lin}}}(x, x')\\
&  = \abs*{\eta(x) - \eta(x')}\min\curl*{W\norm*{x - x'}_p, 1}\\
& \geq \abs*{\eta(x) - \eta(x')}\min\curl*{W\gamma, 1}
\end{align*}
which implies that for any $h \in \sH_{\mathrm{lin}}$ and $(x, x')$ such that $\norm*{x - x'}_p > \gamma$,
\begin{align*}
\Delta\sC_{\lbi_{\Phi_{\mathrm{hinge}}},\sH_{\mathrm{lin}}}(h, x, x') \geq \min\curl*{W\gamma, 1}\Delta\sC_{\labsbi,\sH}(h, x, x').
\end{align*}
Thus, by Theorem~\ref{Thm:bound_Psi} or Theorem~\ref{Thm:bound_Gamma}, setting $\e  =  0$ yields the $\sH_{\mathrm{lin}}$-consistency bound for $\lbi_{\Phi_{\mathrm{hinge}}}$, valid for all $h  \in \sH_{\mathrm{lin}}$:
\begin{align*}
     \sR_{\labsbi}(h) -  \sR_{\labsbi}^*(\sH_{\mathrm{lin}})
     \leq \frac{\sR_{\lbi_{\Phi_{\mathrm{hinge}}}}(h) -  \sR_{\lbi_{\Phi_{\mathrm{hinge}}}}^*(\sH_{\mathrm{lin}}) + \sM_{\lbi_{\Phi_{\mathrm{hinge}}}}(\sH_{\mathrm{lin}})}{\min\curl*{W\gamma, 1}} - \sM_{\labsbi}(\sH_{\mathrm{lin}}).
\end{align*}

\paragraph{Proof for \texorpdfstring{$\lbi_{\Phi_{\mathrm{exp}}}$}{exp}.} For the exponential loss function $\Phi_{\mathrm{exp}}(u)\colon =
e^{-u}$, for all $h \in \sH_{\mathrm{lin}}$ and $(x, x')$ such that
$\norm*{x - x'}_p > \gamma$,
\begin{equation*}
\begin{aligned}
& \sC_{\lbi_{\Phi_{\mathrm{exp}}}}(h, x, x')\\
&  = \eta(x)(1 - \eta(x')) \Phi_{\mathrm{exp}}(h(x) - h(x')) + \eta(x')(1 - \eta(x))\Phi_{\mathrm{exp}}(h(x') - h(x))\\
&  = \eta(x)(1 - \eta(x'))e^{-h(x) + h(x')} + \eta(x')(1 - \eta(x))e^{h(x) - h(x')}.
\end{aligned}
\end{equation*}
Then,
\begin{align*}
\sC^*_{\lbi_{\Phi_{\mathrm{exp}}},\sH_{\mathrm{lin}}}(x, x')
& =  \inf_{h \in\sH_{\mathrm{lin}}}\sC_{\lbi_{\Phi_{\mathrm{exp}}}}(h, x, x')\\
& =  
\begin{cases}
2\sqrt{\eta(x)\eta(x')(1 - \eta(x))(1 - \eta(x'))} \\
\text{if }\frac{1}{2}\abs*{\log\frac{\eta(x)(1 - \eta(x'))}{\eta(x')(1 - \eta(x))}}\leq W\norm*{x - x'}_p\\
\max\curl*{\eta(x)(1 - \eta(x')),\eta(x')(1 - \eta(x))}e^{-W\norm*{x - x'}_p}\\
\qquad + \min\curl*{\eta(x)(1 - \eta(x')),\eta(x')(1 - \eta(x))}e^{W\norm*{x - x'}_p} &\\
\text{if } \frac{1}{2}\abs*{\log\frac{\eta(x)(1 - \eta(x'))}{\eta(x')(1 - \eta(x))}} >  W\norm*{x - x'}_p.
\end{cases}
\end{align*}
The $\paren*{\lbi_{\Phi_{\mathrm{exp}}},\sH_{\mathrm{lin}}}$-minimizability gap is:
\begin{equation*}
\begin{aligned}
& \sM_{\lbi_{\Phi_{\mathrm{exp}}}}(\sH_{\mathrm{lin}})\\
&  =  \sR_{\lbi_{\Phi_{\mathrm{exp}}}}^*(\sH_{\mathrm{lin}}) - 
\mathbb{E}_{(X, X')}\bracket*{\sC^*_{\lbi_{\Phi_{\mathrm{exp}}}, \sH_{\mathrm{lin}}}(x, x')}\\
&  =  \sR_{\lbi_{\Phi_{\mathrm{exp}}}}^*(\sH_{\mathrm{lin}}) - 
\mathbb{E}_{(X, X')}\bracket*{2\sqrt{\eta(x)\eta(x')(1 - \eta(x))(1 - \eta(x'))}\1_{\frac{1}{2}\abs*{\log\frac{\eta(x)(1 - \eta(x'))}{\eta(x')(1 - \eta(x))}}\leq W\norm*{x - x'}_p}}\\
& \qquad - \mathbb{E}_{(X, X')}\bracket*{\bracket*{\max\curl*{\eta(x),\eta(x')} - \eta(x)\eta(x')}e^{-W\norm*{x - x'}_p}\1_{\frac{1}{2}\abs*{\log\frac{\eta(x)(1 - \eta(x'))}{\eta(x')(1 - \eta(x))}} >  W\norm*{x - x'}_p}}\\
& \quad \qquad - \mathbb{E}_{(X, X')}\bracket*{\bracket*{\min\curl*{\eta(x),\eta(x')} - \eta(x)\eta(x')}e^{W\norm*{x - x'}_p}\1_{\frac{1}{2}\abs*{\log\frac{\eta(x)(1 - \eta(x'))}{\eta(x')(1 - \eta(x))}} >  W\norm*{x - x'}_p}}.
\end{aligned}
\end{equation*}
Therefore, $\forall h  \in \wt\sH_{\mathrm{lin}}(x, x')\bigcup \mathring{\sH}_{\mathrm{lin}}(x, x')$,
\begin{align*} & \Delta\sC_{\lbi_{\Phi_{\mathrm{exp}}},\sH_{\mathrm{lin}}}(h, x, x')\\
& \geq  \inf_{h \in\wt\sH_{\mathrm{lin}}(x, x')\bigcup \mathring{\sH}_{\mathrm{lin}}(x, x')}\sC_{\lbi_{\Phi_{\mathrm{exp}}}}(h, x, x') - \sC^*_{\lbi_{\Phi_{\mathrm{exp}}},\sH_{\mathrm{lin}}}(x, x')\\
&  =  \eta(x)(1 - \eta(x'))e^{-0} + \eta(x')(1 - \eta(x))e^{0} - \sC^*_{\lbi_{\Phi_{\mathrm{exp}}},\sH_{\mathrm{lin}}}(x, x')\\
&  =  
\begin{cases}
\eta(x)(1 - \eta(x')) + \eta(x')(1 - \eta(x)) -  2\sqrt{\eta(x)\eta(x')(1 - \eta(x))(1 - \eta(x'))} \\
\text{if } \frac{1}{2}\abs*{\log\frac{\eta(x)(1 - \eta(x'))}{\eta(x')(1 - \eta(x))}}\leq W\norm*{x - x'}_p\\
\bracket*{\max\curl*{\eta(x),\eta(x')} - \eta(x)\eta(x')}\paren*{1 - e^{-W\norm*{x - x'}_p}}\\
\qquad + \bracket*{\min\curl*{\eta(x),\eta(x')} - \eta(x)\eta(x')}\paren*{1 - e^{W\norm*{x - x'}_p}} \\
\text{if } \frac{1}{2}\abs*{\log\frac{\eta(x)(1 - \eta(x'))}{\eta(x')(1 - \eta(x))}} >  W\norm*{x - x'}_p
\end{cases}\\
& \geq
\begin{cases}
\eta(x)(1 - \eta(x')) + \eta(x')(1 - \eta(x)) -  2\sqrt{\eta(x)\eta(x')(1 - \eta(x))(1 - \eta(x'))} \\
\text{if } \frac{1}{2}\abs*{\log\frac{\eta(x)(1 - \eta(x'))}{\eta(x')(1 - \eta(x))}}\leq W\gamma\\
\bracket*{\max\curl*{\eta(x),\eta(x')} - \eta(x)\eta(x')}\paren*{1 - e^{-W\gamma}}\\
\qquad + \bracket*{\min\curl*{\eta(x),\eta(x')} - \eta(x)\eta(x')}\paren*{1 - e^{W\gamma}} \\
\text{if } \frac{1}{2}\abs*{\log\frac{\eta(x)(1 - \eta(x'))}{\eta(x')(1 - \eta(x))}} >  W\gamma
\end{cases}\\
&  = 
\begin{cases}
\paren*{\frac{\eta(x)(1 - \eta(x')) - \eta(x')(1 - \eta(x))}{\sqrt{\eta(x)(1 - \eta(x'))} + \sqrt{\eta(x')(1 - \eta(x))}}}^2 \\\text{if } \frac{1}{2}\abs*{\log\frac{\eta(x)(1 - \eta(x'))}{\eta(x')(1 - \eta(x))}}\leq W\gamma\\
\frac{\eta(x)(1 - \eta(x')) + \eta(x')(1 - \eta(x))}{2}(2 - e^{-W\gamma} - e^{W\gamma}) +  \frac{1}{2}\abs*{\eta(x) - \eta(x')}\paren*{e^{W\gamma} - e^{-W\gamma}}
 \\\text{if } \frac{1}{2}\abs*{\log\frac{\eta(x)(1 - \eta(x'))}{\eta(x')(1 - \eta(x))}} >  W\gamma
\end{cases}\\
& \geq
\min\curl*{\paren*{\eta(x) - \eta(x')}^2,\paren*{\frac{e^{2W\gamma} + 1}{e^{2W\gamma} - 1}}\, \abs*{\eta(x) - \eta(x')}}
\end{align*}
which implies that for any $h \in \sH_{\mathrm{lin}}$ and $(x, x')$ such that $\norm*{x - x'}_p > \gamma$,
\begin{align*}
\Delta\sC_{\lbi_{\Phi_{\mathrm{exp}}},\sH_{\mathrm{lin}}}(h, x, x') \geq \Psi_{\rm{exp}}\paren*{\Delta\sC_{\labsbi,\sH}(h, x, x')}.
\end{align*}
where $\Psi_{\rm{exp}}$ is the increasing function on $[0,2]$ defined by
\begin{align*}
\forall t \in[0, 1], \quad 
\Psi_{\rm{exp}}(t) = \min\curl*{t^2,\paren*{\frac{e^{2W\gamma} + 1}{e^{2W\gamma} - 1}}\, t}.
\end{align*}
Thus, by Theorem~\ref{Thm:bound_Psi} or Theorem~\ref{Thm:bound_Gamma}, setting $\e  =  0$ yields the $\sH_{\mathrm{lin}}$-consistency bound for $\lbi_{\Phi_{\mathrm{exp}}}$, valid for all $h  \in \sH_{\mathrm{lin}}$:
\begin{align*}
     \sR_{\labsbi}(h) -  \sR_{\labsbi}^*(\sH_{\mathrm{lin}})
     \leq  \Gamma_{\Phi_{\mathrm{exp}}}\paren*{\sR_{\lbi_{\Phi_{\mathrm{exp}}}}(h) -  \sR_{\lbi_{\Phi_{\mathrm{exp}}}}^*(\sH_{\mathrm{lin}}) + \sM_{\lbi_{\Phi_{\mathrm{exp}}}}(\sH_{\mathrm{lin}})} - \sM_{\labsbi}(\sH_{\mathrm{lin}}).
\end{align*}
where $\Gamma_{\Phi_{\mathrm{exp}}}(t) = \max\curl*{\sqrt{t},\paren*{\frac{e^{2W\gamma} - 1}{e^{2W\gamma} + 1}}\, t}$.

\paragraph{Proof for \texorpdfstring{$\lbi_{\Phi_{\mathrm{sig}}}$}{sig}.} For the sigmoid loss function $\Phi_{\mathrm{sig}}(u)\colon = 1 -
\tanh(k u),~k > 0$, for all $h \in \sH_{\mathrm{lin}}$ and $(x, x')$
such that $\norm*{x - x'}_p > \gamma$,
\begin{equation*}
\begin{aligned}
& \sC_{\lbi_{\Phi_{\mathrm{sig}}}}(h, x, x')\\
&  = \eta(x)(1 - \eta(x')) \Phi_{\mathrm{sig}}(h(x) - h(x')) + \eta(x')(1 - \eta(x))\Phi_{\mathrm{sig}}(h(x') - h(x))\\
&  = \eta(x)(1 - \eta(x'))\paren*{1 - \tanh\paren*{k\bracket*{h(x) - h(x')}}} + \eta(x')(1 - \eta(x))\paren*{1 + \tanh\paren*{k\bracket*{h(x) - h(x')}}}
\end{aligned}
\end{equation*}
Then,
\begin{align*}
&\sC^*_{\lbi_{\Phi_{\mathrm{sig}}},\sH_{\mathrm{lin}}}(x, x')\\
&=  \inf_{h \in\sH_{\mathrm{lin}}}\sC_{\lbi_{\Phi_{\mathrm{sig}}}}(h, x, x')\\
&=  \eta(x)(1 - \eta(x')) + \eta(x')(1 - \eta(x)) - \abs*{\eta(x) - \eta(x')}\tanh\paren*{kW\norm*{x - x'}_p}.
\end{align*}
The $\paren*{\lbi_{\Phi_{\mathrm{sig}}},\sH_{\mathrm{lin}}}$-minimizability gap is
\begin{equation*}
\begin{aligned}
&\sM_{\lbi_{\Phi_{\mathrm{sig}}}}(\sH_{\mathrm{lin}})\\
&=  \sR_{\lbi_{\Phi_{\mathrm{sig}}}}^*(\sH_{\mathrm{lin}})\\
&\qquad - \mathbb{E}_{(X, X')}\bracket*{\eta(x)(1 - \eta(x')) + \eta(x')(1 - \eta(x)) - \abs*{\eta(x) - \eta(x')}\tanh\paren*{kW\norm*{x - x'}_p}}.
\end{aligned}
\end{equation*}
Therefore, $\forall h  \in \wt\sH_{\mathrm{lin}}(x, x')\bigcup \mathring{\sH}_{\mathrm{lin}}(x, x')$,
\begin{align*} & \Delta\sC_{\lbi_{\Phi_{\mathrm{sig}}},\sH_{\mathrm{lin}}}(h, x, x')\\
& \geq  \inf_{h \in\wt\sH_{\mathrm{lin}}(x, x')\bigcup \mathring{\sH}_{\mathrm{lin}}(x, x')}\sC_{\lbi_{\Phi_{\mathrm{sig}}}}(h, x, x') - \sC^*_{\lbi_{\Phi_{\mathrm{sig}}},\sH_{\mathrm{lin}}}(x, x')\\
&  =  \eta(x)(1 - \eta(x')) + \eta(x')(1 - \eta(x)) - \sC^*_{\lbi_{\Phi_{\mathrm{sig}}},\sH_{\mathrm{lin}}}(x, x')\\
&  = \abs*{\eta(x) - \eta(x')}\tanh\paren*{kW\norm*{x - x'}_p}\\
& \geq \abs*{\eta(x) - \eta(x')}\tanh\paren*{kW\gamma}
\end{align*}
which implies that for any $h \in \sH_{\mathrm{lin}}$ and $(x, x')$ such that $\norm*{x - x'}_p > \gamma$,
\begin{align*}
\Delta\sC_{\lbi_{\Phi_{\mathrm{sig}}},\sH_{\mathrm{lin}}}(h, x, x') \geq \tanh\paren*{kW\gamma}\Delta\sC_{\labsbi,\sH}(h, x, x').
\end{align*}
Thus, by Theorem~\ref{Thm:bound_Psi} or Theorem~\ref{Thm:bound_Gamma}, setting $\e  =  0$ yields the $\sH_{\mathrm{lin}}$-consistency bound for $\lbi_{\Phi_{\mathrm{sig}}}$, valid for all $h  \in \sH_{\mathrm{lin}}$:
\begin{align*}
     \sR_{\labsbi}(h) -  \sR_{\labsbi}^*(\sH_{\mathrm{lin}})
     \leq \frac{\sR_{\lbi_{\Phi_{\mathrm{sig}}}}(h) -  \sR_{\lbi_{\Phi_{\mathrm{sig}}}}^*(\sH_{\mathrm{lin}}) + \sM_{\lbi_{\Phi_{\mathrm{sig}}}}(\sH_{\mathrm{lin}})}{\tanh\paren*{kW\gamma}} - \sM_{\labsbi}(\sH_{\mathrm{lin}}).
\end{align*}
\end{proof}

\section{Negative results for general pairwise
  ranking (Proof of Theorem~\ref{Thm:negative-general})}
\label{app:general-negative}

\Negative*
\begin{proof}
Assume $x_0 \in \sX$ is an interior point and $h_0 = 0 \in \sH$. By the
assumption that $x_0$ is an interior point and $\sH$ is equicontinuous
at $x_0$, for any $\e > 0$, we are able to take $x' \neq x_0 \in \sX$ such
that $\abs*{h(x') - h(x_0)} <  \e$ for all $h \in \sH$. Consider the
distribution that supports on $\curl*{(x_0,x')}$ with
$\eta(x_0,x') = 0$. Then, for any $h \in \sH$,
\begin{align*}
\sR_{\lrank}(h) = \sC_{\lrank}(h, x_0,x')
 =  \mathds{1}_{h(x')\geq h(x_0)} \geq 0,
\end{align*}
where the equality can be achieved for some $h \in \sH$ since $\sH$ is regular for general pairwise ranking. Therefore, 
\begin{align*}
\sR_{\lrank}^*(\sH) = \sC^*_{\lrank}(\sH,x_0,x') =  \inf_{h \in \sH}\sC_{\lrank}(h, x_0,x') = 0.
\end{align*}
Note $\sR_{\lrank}(h_0) = 1$. 
For the surrogate loss $\sfL_{\Phi}$, for any $h \in \sH$,
\begin{align*}
\sR_{\sfL_{\Phi}}(h) = \sC_{\sfL_{\Phi}}(h, x_0,x')
 = \Phi\paren*{h(x_0) - h(x')} \in \bracket*{\Phi(\e),\Phi(-\e)}
\end{align*}
since $\abs*{h(x') - h(x_0)} <  \e$ and $\Phi$ is non-increasing. Therefore,
\begin{align*}
\sR_{\sfL_{\Phi}}^*(\sH) = \sC^*_{\sfL_{\Phi}}(\sH,x_0,x')\geq \Phi(\e).
\end{align*}
Note $\sR_{\sfL_{\Phi}}(h_0) = \Phi(0)$. If for some function $f$ that
is non-decreasing and continuous at $0$, the bound holds, then, we
obtain for any $h \in \sH$ and $\e > 0$,
\begin{align*}
\sR_{\lrank}(h) - 0\leq  f\paren*{\sR_{\sfL_{\Phi}}(h) - \sR_{\sfL_{\Phi}}^*(\sH)}\leq f\paren*{\sR_{\sfL_{\Phi}}(h) - \Phi(\e)}.
\end{align*}
Let $h = h_0$, then $f\paren*{\Phi(0) - \Phi(\e)}\geq 1$ for any
$\e > 0$. Take $\e\to 0$, we obtain $f(0)\geq 1$ using the fact that
$\Phi$ and $f$ are both continuous at $0$. Since $f$ is
non-decreasing, for any $t \in [0, 1]$, $f(t)\geq 1$.
\end{proof}

\section{Negative results for bipartite ranking (Proof of Theorem~\ref{Thm:negative-bi})}
\label{app:negative-bi}

\NegativeBi*
\begin{proof}
Assume $x_0 \in \sX$ is an interior point and $h_0 = 0 \in \sH$. By the
assumption that $x_0$ is an interior point and $\sH$ is equicontinuous
at $x_0$, for any $\e > 0$, we are able to take $x' \neq x_0 \in \sX$ such
that $\abs*{h(x') - h(x_0)} <  \e$ for all $h \in \sH$. Consider the
distribution that supports on $\curl*{x_0, x'}$ with $\eta(x_0) = 1$ and
$\eta(x') = 0$. Then, for any $h \in \sH$,
\begin{align*}
\sR_{\lrankbi}(h) = \sC_{\lrankbi}(h, x_0,x')
 =  \mathds{1}_{h(x_0) <  h(x')} + \frac{1}{2}\mathds{1}_{h(x_0) =  h(x')} \geq 0,
\end{align*}
where the equality can be achieved for some $h \in \sH$ since $\sH$ is regular for bipartite ranking. Therefore, 
\begin{align*}
\sR_{\lrank}^*(\sH) = \sC^*_{\lrank}(\sH,x_0,x') =  \inf_{h \in \sH}\sC_{\lrank}(h, x_0,x') = 0.
\end{align*}
Note $\sR_{\lrank}(h_0) = \frac{1}{2}$. 
For the surrogate loss $\sfL_{\Phi}$, for any $h \in \sH$,
\begin{align*}
\sR_{\sfL_{\Phi}}(h) = \sC_{\sfL_{\Phi}}(h, x_0,x')
 = \Phi\paren*{h(x_0) - h(x')} \in \bracket*{\Phi(\e),\Phi(-\e)}
\end{align*}
since $\abs*{h(x') - h(x_0)} <  \e$ and $\Phi$ is non-increasing. Therefore,
\begin{align*}
\sR_{\sfL_{\Phi}}^*(\sH) = \sC^*_{\sfL_{\Phi}}(\sH,x_0,x')\geq \Phi(\e).
\end{align*}
Note $\sR_{\sfL_{\Phi}}(h_0) = \Phi(0)$. If for some function $f$ that
is non-decreasing and continuous at $0$, the bound holds, then, we
obtain for any $h \in \sH$ and $\e > 0$,
\begin{align*}
  \sR_{\lrank}(h) - 0
  \leq  f\paren*{\sR_{\sfL_{\Phi}}(h) - \sR_{\sfL_{\Phi}}^*(\sH)}\leq f\paren*{\sR_{\sfL_{\Phi}}(h) - \Phi(\e)}.
\end{align*}
Let $h = h_0$, then $f\paren*{\Phi(0) - \Phi(\e)}\geq \frac{1}{2}$ for any
$\e > 0$. Take $\e\to 0$, we obtain $f(0)\geq \frac{1}{2}$ using the fact
that $\Phi$ and $f$ are both continuous at $0$. Since $f$ is
non-decreasing, for any $t \in [0, 1]$, $f(t)\geq \frac{1}{2}$.
\end{proof}

\end{document}